\PassOptionsToPackage{table}{xcolor}
\documentclass{article}
\usepackage{tamedflow_preprint,times}

\usepackage{amsmath,amsfonts,bm}

\def\eqref#1{equation~\ref{#1}}

\def\1{\bm{1}}

\DeclareMathAlphabet{\mathsfit}{\encodingdefault}{\sfdefault}{m}{sl}
\SetMathAlphabet{\mathsfit}{bold}{\encodingdefault}{\sfdefault}{bx}{n}

\usepackage[T1]{fontenc}
\usepackage[utf8]{inputenc}
\usepackage{hyperref}
\hypersetup{colorlinks=true,citecolor=blue,linkcolor=blue,urlcolor=blue,pdfborder={0 0 0}}

\usepackage{url}
\usepackage{enumitem}
\usepackage{booktabs}
\usepackage{amsfonts,amsmath,amssymb,amsthm}
\usepackage{nicefrac}
\usepackage{microtype}
\usepackage{graphicx}
\usepackage{algorithm}
\usepackage{algpseudocode}
\usepackage{multirow}
\usepackage{adjustbox}
\usepackage{caption}
\usepackage{framed}
\usepackage{wrapfig}
\usepackage{xcolor}
\usepackage{titlesec}

\newtheorem{prop}{Proposition}[section]
\title{TamedFlow: Trajectory-Aware Distillation\\for Teacher-Assisted Diffusion Sampling}
\author{%
Shuaiqi Cheng\textsuperscript{1,*}\quad
Jiahao Zhang\textsuperscript{1,*}\quad
Boyuan Gu\textsuperscript{2}\quad
Yuxi Chen\textsuperscript{3}\\[4pt]
Haiying Xu\textsuperscript{1}\quad
Zexin Lu\textsuperscript{4}\quad
Huiyu Zhou\textsuperscript{5}\quad
Xuming Hu\textsuperscript{1,3}\\[9pt]
{\small\textsuperscript{1}The Hong Kong University of Science and Technology (Guangzhou)}\\[2pt]
{\small\textsuperscript{2}University of Electronic Science and Technology of China}\\[2pt]
{\small\textsuperscript{3}The Hong Kong University of Science and Technology}\\[2pt]
{\small\textsuperscript{4}Huawei Technologies Ltd.}\\[2pt]
{\small\textsuperscript{5}Guangxi Zhuang Autonomous Region Information Center}\\[6pt]
{\small\textsuperscript{*}These authors contributed equally.}%
}

\begin{document}

\maketitle

\begin{abstract}
Diffusion Transformers (DiTs) deliver high-fidelity image and video generation, but inference requires many sequential model evaluations. Two acceleration strategies are increasingly important: model compression makes each evaluation cheaper, while adaptive inference avoids selected teacher computation through reuse, forecasting, or speculative execution. Their combination, however, is brittle: directly applying the evaluated adaptive methods to a compressed student can degrade quality, an observed failure mode we call the \textbf{Combination Gap}. We link this gap to a mismatch between pointwise distillation and short-horizon execution: a student trained at teacher states need not remain reliable when its own predictions are chained. We propose \textbf{TamedFlow}, which trains a compressed DiT for this deployed role through cross-architecture feature alignment, a near-identity Velocity Alignment Head, and randomized rollout supervision. At inference, a lightweight Learnable Trajectory Estimator routes completed drafts: accepted drafts bypass the teacher, while rejected drafts roll back and receive one teacher rescue step. Distilling DiT-XL/2 into DiT-B/2, TamedFlow reaches FID $2.58$ at $2.25\times$ measured speedup over DDIM-50 on ImageNet $256\times256$. Ablations support trajectory supervision, and routing controls favor LTE over random decisions at nearly matched latency. Evaluations on ImageNet $512\times512$, MS-COCO, and Latte/FaceForensics extend the empirical scope.
\end{abstract}

\section{Introduction}

Diffusion Transformers (DiTs) have become a leading backbone for high-fidelity visual generation~\citep{peebles2023dit}, yet iterative denoising imposes high inference latency. Each sample requires dozens of sequential forward passes, with costs scaling rapidly with model parameters, resolution, and sequence length. Existing acceleration strategies attack complementary aspects of this cost: step-reduction methods shorten the solver schedule~\citep{song2023consistency,lcm2023,meanflow2025,kim2025ard}; parameter compression lowers per-step computation; and adaptive inference reuses or forecasts intermediate activations~\citep{deepcache2024,fora2024,taylorseer2025,speca2025}. Combining parameter compression with adaptive execution is intuitively appealing, promising cheaper evaluations while skipping redundant computation. All headline experiments below use fixed-step DDIM integration to evaluate compression and adaptive execution under a unified solver convention.

In practice, this composition can break down. In our controlled diagnostic (Appendix Table~\ref{tab:same_student_validation}), attaching the evaluated adaptive reuse methods directly to a compressed student degrades sample quality, inflating ImageNet $256\times256$ FID from $5.56$ to between $10.40$ and $13.18$. We term this failure mode the \textbf{Combination Gap}. This result exposes an operational mismatch: a student trained for pointwise agreement need not remain reliable after its own predictions alter the states encountered later. Rather than applying internal reuse unchanged to the compact model, we let the student generate short trajectory drafts and reserve teacher intervention for drafts predicted to deviate.

\begin{figure}[t]
  \centering
  \includegraphics[width=0.8\linewidth]{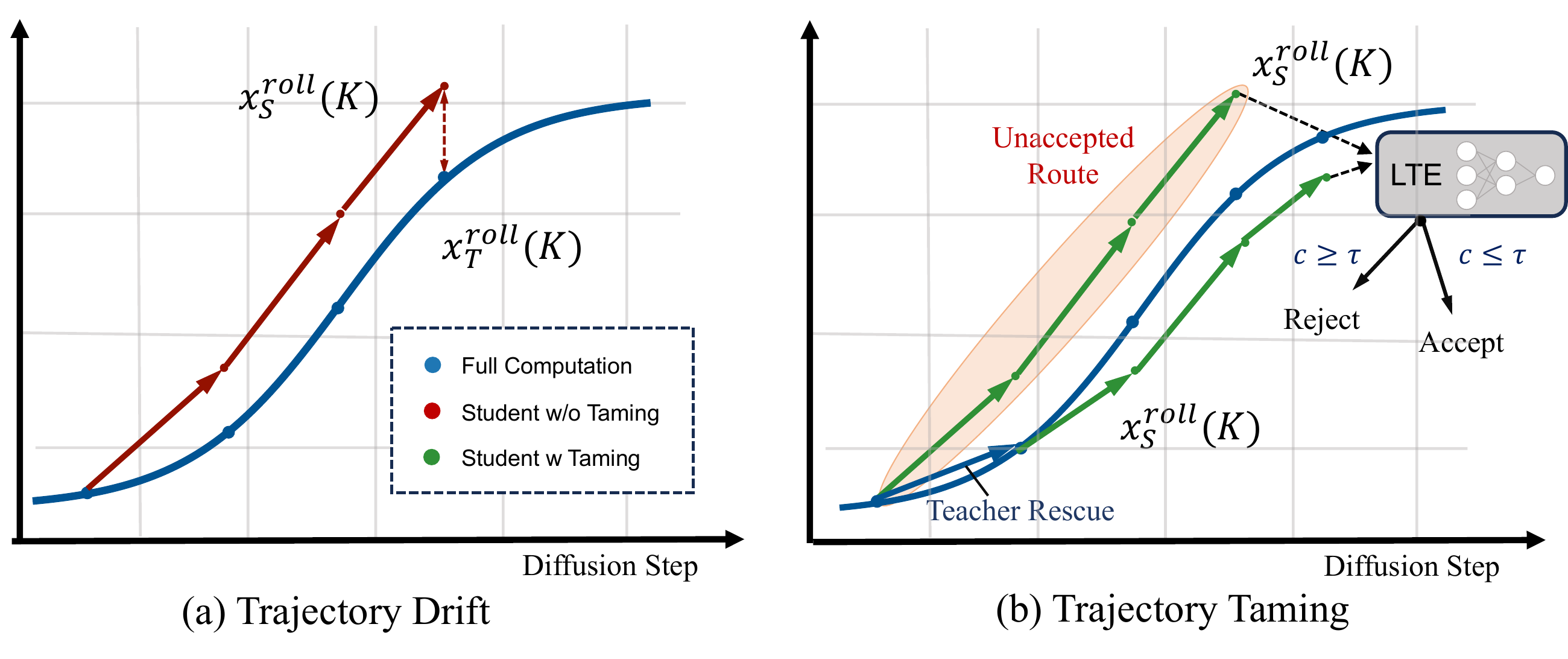}
  \caption{\textbf{The Combination Gap and TamedFlow.} Directly applying adaptive reuse to a parameter-compressed student severely degrades quality in our diagnostic. TamedFlow instead trains the compressed student for short-horizon trajectory agreement. A 4M-parameter LTE routes drafts: accepted drafts bypass the teacher; rejected drafts roll back for one teacher rescue step.}
  \label{fig:motivation}
\end{figure}

Chaining student predictions into drafts, however, introduces compounding integration drift. Proposition~\ref{prop:curvature_trap} characterizes this through an Euler rollout bound: spatial sensitivity can amplify local student--teacher velocity mismatch over sequential steps. To stabilize drafting against this amplification, we introduce \textbf{TamedFlow}, combining trajectory-aware distillation with selective teacher rescue (Figure~\ref{fig:motivation}). During training, cross-architecture feature alignment (AlignNet) grounds student representations in teacher layers; a near-identity Velocity Alignment Head (VAH) absorbs residual velocity errors; and randomized $K$-step rollout supervision ($K \sim \mathcal{U}(1,5)$) penalizes both intermediate mismatch and accumulated endpoint drift. A two-stage curriculum introduces the multi-step objectives only after feature alignment.

At inference, TamedFlow uses an asymmetric execution policy. A compact Learnable Trajectory Estimator (LTE, 4M parameters) scores completed drafts from pooled student tokens and timestep embeddings without teacher access. If confidence exceeds threshold $\tau$, the draft is accepted, advancing $K$ steps while bypassing the teacher. If rejected, the sampler rolls back to the pre-draft state and advances one step with teacher rescue. Varying $\tau$ selects among measured operating points between throughput and fidelity without sequential teacher verification.

Distilling DiT-XL/2 into a 130M DiT-B/2 student on ImageNet $256\times256$, TamedFlow attains FID $2.58$ at $2.25\times$ measured speedup ($0.464$~s vs. $1.044$~s) over the DDIM-50 teacher. At comparable latency, TamedFlow improves FID over the evaluated training-free adaptive baselines (e.g., FoCa at FID $2.65$, $0.466$~s). Controlled ablations support trajectory supervision and learned routing at nearly matched latency, while evaluations on ImageNet $512\times512$, MS-COCO, and Latte video DiTs extend the empirical scope. In summary, our contributions are:
\begin{itemize}[leftmargin=12pt,itemsep=0.5pt,topsep=0pt,parsep=0pt]
  \item \textbf{Diagnostic and analysis of the Combination Gap.} We document the failure of applying the evaluated adaptive reuse methods to a compressed DiT, and derive a DDIM rollout bound characterizing how sequential drafting can amplify local velocity errors.
  \item \textbf{Trajectory-stabilized speculative distillation.} We introduce TamedFlow, uniting feature alignment, residual velocity correction, and randomized rollout supervision to optimize a compressed DiT for autonomous multi-step drafting.
\item \textbf{Learned routing and deployment evaluation.} We propose an asymmetric execution policy via a 4M-parameter LTE to accept reliable drafts or trigger teacher rescue, evaluating quality--latency trade-offs across resolutions, text conditioning, and video alongside physical deployment costs.
\end{itemize}

\section{Related Work}
\label{sec:related}

Prior work on DiT acceleration reduces the number of solver evaluations, the cost of each evaluation, or the teacher computation retained during adaptive execution. TamedFlow studies their interface: whether a compressed student can reliably replace short segments of a teacher trajectory.

\subsection{Step Count Reduction in Diffusion and Flow Models}
\label{sec:step-distillation}

Step reduction lowers the number of solver evaluations. Higher-order solvers improve numerical efficiency for a fixed model~\citep{song2020ddim,lu2022dpmsolverpp}, while distillation changes the model: progressive distillation shortens the schedule~\citep{salimans2022progressive}, consistency models match predictions across noise levels~\citep{song2023consistency,lcm2023}, MeanFlow and Model-Guidance alter the target or absorb guidance~\citep{meanflow2025,modelguidance2025}. Other formulations align distributions or scores, distill autoregressively, transfer trajectories, or train few-step students adversarially~\citep{dmd22024,sid2025,kim2025ard,luo2025tdm,sauer2023add}. These methods optimize a sparse standalone sampler deployed entirely on its own. TamedFlow instead retains a multi-step teacher as a reliable safeguard and uses the student to propose short trajectory segments whose acceptance is decided adaptively online.

A related line studies trajectory geometry. Rectified flow and InstaFlow construct straighter paths~\citep{liu2024insta}; interpolant transformers unify diffusion and flow parametrizations~\citep{ma2024sit}; and guided path sampling corrects trajectories during integration~\citep{gps2025}. TCD and Hyper-SD learn consistency or trajectory segments for few-step sampling~\citep{zheng2024tcd,ren2024hypersd}. TamedFlow keeps the teacher path as the reference, trains a cross-architecture student for selectively accepted rollouts, and retains teacher rescue in the deployed sampler.

\subsection{Per-Step Compute Acceleration and Feature Caching}
\label{sec:per-step}

Per-step methods keep the schedule while avoiding full recomputation. Feature-reuse methods cache activations across timesteps or blocks~\citep{deepcache2024,delta-dit2024,fora2024,toca2025,duca2024,wimbauer2023blockcache}; forecasting methods extrapolate future features~\citep{taylorseer2025,speca2025,lesa2026,harmonica2025}; and learned schedulers select policies per sample~\citep{rapid32026}. These methods exploit local regularity in intermediate computations. Our same-student diagnostic shows that their evaluated configurations do not transfer directly to the compressed student, but does not isolate a universal caching failure mechanism.

Most reuse methods leave the model unchanged and decide which internal computations can be skipped. TamedFlow instead trains the student under rollout supervision and learns whether an entire student-generated segment can advance without a teacher call. DisCa~\citep{disca2026} also studies incompatibility between step-distilled video DiTs and feature caches, but repairs the cache side with a predictor and Restricted MeanFlow schedule. TamedFlow acts on the draft student instead.

\subsection{Speculative Verification in Generative Paradigms}
\label{sec:speculative}

Speculative decoding verifies proposed tokens with a target model~\citep{leviathan2023fast}; causal masking permits parallel verification and exact correction. Diffusion lacks this structure because each teacher state depends on the preceding prediction. Replaying a $K$-step draft exactly would require sequential teacher evaluations and remove much of the saving. SpecDiff~\citep{specdiff2025} studies self-speculative diffusion under this constraint. TamedFlow replaces exact verification with learned trajectory estimation. The LTE scores student-internal features after a draft using rollout agreement as its training target. Accepted drafts avoid the teacher; rejected drafts return to the pre-draft state for one teacher step. Routing is therefore inexpensive and online, but approximate and not distribution-preserving by construction.

\section{Method}
\label{sec:method}
\subsection{Preliminaries: Diffusion Models, DiT, and Flow Matching}
\label{sec:preliminaries}

Diffusion models~\citep{ho2020denoising} learn to reverse a gradual noising process, so generation requires a sequence of state-dependent model evaluations. A Diffusion Transformer (DiT)~\citep{peebles2023dit} implements the denoiser with a transformer over latent patches and injects timestep and condition information through adaptive layer normalization. Our method is expressed in velocity form because both its trajectory loss and its stability analysis act on the vector field followed during sampling.

Flow matching~\citep{lipman2022flowmatching} casts generation as a probability-path ODE. With $x_0\sim\mathcal{N}(0,I)$, $x_1\sim p_{\mathrm{data}}$, and the linear interpolant $x_t=(1-t)x_0+tx_1$, the conditional velocity $x_1-x_0$ supplies the regression target:
\begin{equation}
\label{eq:fm}
\mathcal{L}_{\mathrm{FM}} = \mathbb{E}_{t,x_0,x_1}\!\left[\bigl\|v_\theta(x_t,t) - (x_1-x_0)\bigr\|_2^2\right],\qquad t\sim\mathcal{U}[0,1].
\end{equation}
We use forward flow time, from noise at $t=0$ to data at $t=1$. The main experiments use the stated DDIM schedules; the Euler discretization below is an analysis device and the rollout rule used to define the trajectory objectives.

\subsection{Spatial Error Amplification and Rollout Stability}
\label{sec:curvature_trap}

Pointwise velocity matching constrains the student at sampled states but not the states reached after its own predictions are chained. The following proposition separates the local mismatch that initiates drift from the spatial sensitivity that amplifies it. It applies to sequential Euler drafts and does not analyze internal feature caching or forecasting.

\begin{prop}[$K$-Step Trajectory Alignment Bound]
\label{prop:curvature_trap}
Let $\tilde v_S\equiv v_S^{\mathrm{aligned}}$ denote the deployed student velocity field, and let $x_S^{\mathrm{roll}}(K)$ and $x_T^{\mathrm{roll}}(K)$ be the aligned-student and teacher endpoints after $K$ sequential Euler steps from state $x_t$ with fixed step size $h$. With time arguments suppressed at the common $t_k=t+kh$, define the baseline velocity mismatch $\delta_v^{(k)}\triangleq \tilde v_S(x_T^{(k)})-v_T(x_T^{(k)})$, with $\|\delta_v\|_\infty \triangleq \max_k \|\delta_v^{(k)}\|_2$. Let $L_T\triangleq\sup_{x,t}\|J_{v_T}(x,t)\|_2$ and $\Delta J\triangleq\sup_{x,t}\|J_{\tilde v_S}(x,t)-J_{v_T}(x,t)\|_2$. Assuming both velocity fields are continuously differentiable on a region containing each segment between paired student and teacher states and $C\triangleq L_T+\Delta J<\infty$, the spatial deviation satisfies the discrete Gr\"onwall upper bound:
\begin{equation}
\label{eq:rollout_bound}
\|x_S^{\mathrm{roll}}(K)-x_T^{\mathrm{roll}}(K)\|_2
\leq \frac{\|\delta_v\|_\infty}{C}\big((1+hC)^K-1\big),
\end{equation}
where $h>0$ is the fixed Euler integration step size; for $C=0$, the right-hand side is interpreted as $Kh\|\delta_v\|_\infty$. For $KhC\ll1$, this implies the few-step expansion $Kh\|\delta_v\|_\infty+\mathcal{O}(K^2h^2C\|\delta_v\|_\infty)$.
\end{prop}

The recurrence $\|e_{k+1}\|_2\leq(1+hC)\|e_k\|_2+h\|\delta_v\|_\infty$, unrolled from $e_0=0$, gives Eq.~\eqref{eq:rollout_bound}; Appendix~\ref{app:proof_prop1} gives the proof. The implication for method design is limited but useful: reducing one-step mismatch is insufficient when the same error is repeatedly propagated, so training should also expose the student to short rollouts. The bound neither establishes a universal quadratic error law nor guarantees that our objectives minimize $C$.

\subsection{Stagewise Training}
\label{sec:tamedflow}

\begin{figure}[t]
  \centering
  \includegraphics[width=0.95\linewidth]{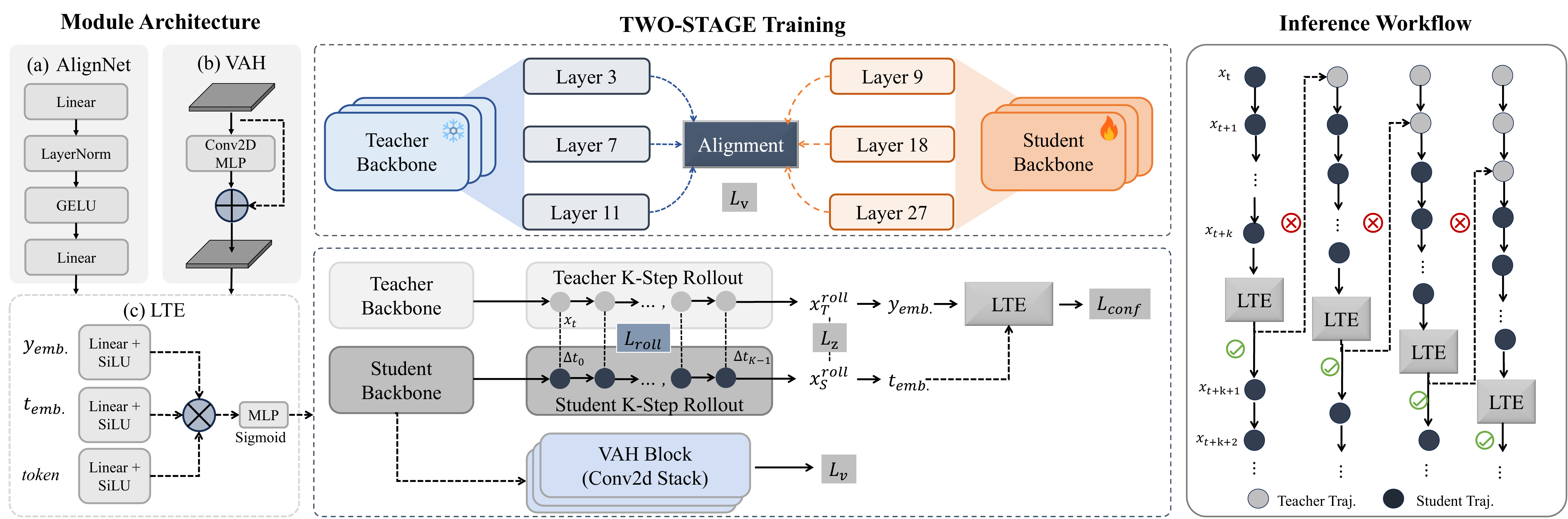}
  \caption{TamedFlow overview. Left: AlignNet, VAH, and LTE attached to the student backbone. Middle: two-stage curriculum. Stage~1 is feature alignment and Stage~2 is rollout supervision. Right: inference routes drafts via the LTE; accepted drafts bypass the teacher, rejected drafts roll back and take one teacher step.}
  \label{fig:pipeline}
\end{figure}

Figure~\ref{fig:pipeline} summarizes AlignNet for feature alignment, VAH for residual velocity correction, and LTE for draft scoring. A two-stage curriculum separates feature initialization from rollout supervision.

\paragraph{Setup and state extraction.} Let $v_T$ and $v_S$ denote the teacher and student velocity fields used by the solver. Pointwise distillation uses $\mathcal{L}_v=\mathbb{E}[D^{-1}\|v_S(x_t,t)-v_T(x_t,t)\|_2^2]$, where $D$ is the latent dimension; Stage~2 applies this loss to the aligned student velocity defined below. The student backbone also exposes hidden tokens from its ordinary forward pass, which the LTE later uses for routing without an additional backbone evaluation.

\paragraph{Stage 1: Feature alignment.} Stage 1 freezes VAH and trains the backbone with velocity and intermediate-feature matching. AlignNet uses 2-layer MLP projectors to pair student layers $\{3,7,11\}$ with teacher layers $\{9,18,27\}$, spanning early through late representations.
\begin{equation}
\label{eq:align}
\mathcal{L}_{\mathrm{align}} = \frac{1}{3}\sum_{l\in\{3,7,11\}}
\frac{\|A_l(f_S^l) - f_T^{\pi(l)}\|_2^2}
     {\|f_T^{\pi(l)}\|_2^2}.
\end{equation}
Here $\pi(3)=9$, $\pi(7)=18$, and $\pi(11)=27$. The Stage~1 objective $\mathcal{L}_{\mathrm{stage1}}=\mathcal{L}_v+\lambda_{\mathrm{align}}\mathcal{L}_{\mathrm{align}}$ initializes a student whose intermediate representation and velocity prediction are both anchored to the teacher before free-running supervision begins.

\paragraph{Stage 2: Rollout supervision.} After Stage 1 stabilizes, we activate rollout-level supervision. To learn a residual velocity correction, we unfreeze the \textbf{Velocity Alignment Head (VAH)}, a convolutional residual head initialized with a small output correction, so $v_S^{\mathrm{aligned}}\approx v_S$:
\begin{equation}
\label{eq:vah}
v_S^{\mathrm{aligned}} = v_S + \mathrm{VAH}(v_S; \phi).
\end{equation}
Stage~2 trains the model for short-horizon use. We compare aligned student and teacher velocities at teacher states and separately penalize the endpoints obtained when each model follows its own rollout:
\begin{equation}
\label{eq:stage2}
\mathcal{L}_{\mathrm{stage2}} = \mathcal{L}_v + \lambda_{\mathrm{roll}}\mathcal{L}_{\mathrm{roll}}(K) + \lambda_z\mathcal{L}_z + \lambda_{\mathrm{conf}}\mathcal{L}_{\mathrm{conf}},
\end{equation}
where $\mathcal{L}_{\mathrm{roll}}(K)=\frac{1}{KD}\sum_{k=0}^{K-1}\|v_S^{\mathrm{aligned}}(x_T^{(k)},t+k\Delta t)-v_T(x_T^{(k)},t+k\Delta t)\|_2^2$ is teacher-state velocity anchoring, whereas $\mathcal{L}_z=D^{-1}\|x_S^{\mathrm{roll}}(K)-x_T^{\mathrm{roll}}(K)\|_2^2$ is the free-running endpoint loss obtained by rolling each field from the same start state on its own states. The endpoint loss backpropagates through the unrolled student steps, while training samples $K$ uniformly from $\{1,2,3,4,5\}$ so the model is not specialized to one horizon. For any fixed teacher-state rollout, $\|\delta_v\|_\infty\leq\sqrt{KD\mathcal{L}_{\mathrm{roll}}}$; this links the empirical loss to the forcing term in Proposition~\ref{prop:curvature_trap}, but is not a uniform guarantee from the average training objective. The \textbf{LTE} is trained with
\begin{equation}
\label{eq:conf_loss}
\mathcal{L}_{\mathrm{conf}} = \mathrm{MSE}\bigl(\mathrm{LTE}(\mathrm{sg}[\bar{f}_S^{\mathrm{last}}, e_{t_{K-1}}, e_y]), \mathrm{sg}[c^\star]\bigr),
\end{equation}
where $\bar{f}_S^{\mathrm{last}}$ denotes mean-pooled hidden tokens from the last student forward in the draft and $\mathrm{sg}[\cdot]$ denotes stop-gradient. Thus the LTE learns from the same type of signal available online, while its loss does not perturb the student backbone.

\subsection{Self-Speculative Execution via Confidence Routing}
\label{sec:inference}

For a $K$-step draft, the target score $c^\star$ combines endpoint rollout deviation with the velocity error at the shared start state:
\begin{equation}
\label{eq:c_star}
c^\star = \exp\!\Bigl(-\alpha\frac{\|x_S^{\mathrm{roll}}(K)-x_T^{\mathrm{roll}}(K)\|_2^2}{D T_z}
- (1-\alpha)\frac{\|v_S^{\mathrm{aligned}}(x_t,t) - v_T(x_t,t)\|_2^2}{D T_v}\Bigr),
\end{equation}
with latent dimension $D$ and temperatures $T_z,T_v$. The score decreases as either discrepancy grows; it is a regression target, not a calibrated probability that a draft preserves sample quality. At inference the LTE predicts this score from student features alone.

\paragraph{LTE inference.} For a clipped draft of length $m$, the LTE uses mean-pooled tokens and the timestep embedding from the last drafting forward at $t_{i+m-1}$, together with the condition representation:
\begin{equation}
\label{eq:lte}
h_{\mathrm{LTE}} = \sigma(g_f)W_1\bar{f}_S + \sigma(g_t)W_2 e_{t_{i+m-1}} + \sigma(g_y)W_3 e_y, \qquad c = \sigma\bigl(w^\top\mathrm{MLP}(h_{\mathrm{LTE}})+b\bigr).
\end{equation}
The LTE has no explicit horizon input, so its score is an empirical routing signal rather than a horizon-specific error guarantee.

\paragraph{Routing.} The student drafts $K$ steps, clipped at the schedule boundary. If $c\geq\tau$, the sampler accepts the entire draft and bypasses the teacher; otherwise it discards the tentative states, restores the pre-draft state, and advances one step with the teacher. Appendix~\ref{app:inference_details} gives the full algorithm.

\section{Experiments}
\label{sec:experiment}

\begin{table}[htbp]
  \centering
  \small
  \setlength{\tabcolsep}{5pt}
  \renewcommand{\arraystretch}{1.1}
  \caption{ImageNet $256\times256$ class-conditional generation with DiT-XL/2.}
  \label{tab:main}
  \scalebox{0.80}{
  \begin{tabular}{lccccc}
    \toprule
    \textbf{Method} & \textbf{Latency (s) $\downarrow$} & \textbf{Speedup} $\uparrow$ & \textbf{FLOPs (T) $\downarrow$} & \textbf{FID $\downarrow$} & \textbf{IS} $\uparrow$ \\
    \midrule
    \textit{DDIM-50 Teacher} & \textit{1.044} & \textit{1.00$\times$} & \textit{23.74} & \textit{2.35} & \textit{237.62} \\
    \textbf{Pure Student}    & 0.180 & 5.80$\times$ & 2.30 & 5.56 {\color{gray}(+3.21)} & 195.07 {\color{gray}(-42.55)} \\
    \midrule
    \multicolumn{6}{c}{\centering \textit{Full-Scale Few-Step Distillation}} \\
    Progressive Distill. (step=4)~\citep{salimans2022progressive} & 0.068 & 15.35$\times$ & $0.474$ & 10.25 {\color{gray}(+7.90)} & 181.58 {\color{gray}(-56.04)} \\
    LCM (step=4)~\citep{lcm2023} & 0.072 & 14.50$\times$ & $0.476$ & 6.81 {\color{gray}(+4.46)} & 202.73 {\color{gray}(-34.89)} \\
    ARD (step=4)~\citep{kim2025ard} ($N=6$) & 0.074 & 14.11$\times$ & $0.480$ & 4.32 {\color{gray}(+1.97)} & 209.03 {\color{gray}(-28.59)} \\
    \midrule
    \multicolumn{6}{c}{\centering \textit{1.5$\times$--2.0$\times$ Speedup Range}} \\
    DDIM-25 & 0.558 & 1.87$\times$ & 11.87 & 3.21 {\color{gray}(+0.86)} & 229.91 {\color{gray}(-7.71)} \\
    $\Delta$-DiT \citep{delta-dit2024} ($N=2$) & 0.684 & 1.53$\times$ & 18.04 & 2.78 {\color{gray}(+0.43)} & 223.47 {\color{gray}(-14.15)} \\
    TaylorSeer \citep{taylorseer2025} ($N=2$) & 0.630 & 1.66$\times$ & 12.35 & 2.47 {\color{gray}(+0.12)} & 230.21 {\color{gray}(-7.41)} \\
    FoCa \citep{foca2025} ($N=2$) & 0.634 & 1.65$\times$ & 12.35 & 2.37 {\color{gray}(+0.02)} & 233.70 {\color{gray}(-3.92)} \\
    SpecDiff~\citep{specdiff2025} & 0.644 & 1.62$\times$ & 14.58 & 2.39 {\color{gray}(+0.04)} & 235.18 {\color{gray}(-2.44)} \\
    \rowcolor{gray!10} \textbf{TamedFlow} ($\tau=0.85$) & \textbf{0.602} & \textbf{1.73$\times$} & $8.53$ & \textbf{2.33} {\color{blue!60!black}(-0.02)} & \textbf{236.15} {\color{blue!60!black}(-1.47)} \\
    \midrule
    \multicolumn{6}{c}{\centering \textit{2.0$\times$--2.5$\times$ Speedup Range}} \\
    DDIM-20 & 0.442 & 2.36$\times$ & 9.49 & 3.87 {\color{gray}(+1.52)} & 219.40 {\color{gray}(-18.22)} \\
    FORA \citep{fora2024} ($N=3$) & 0.482 & 2.17$\times$ & 8.58 & 3.65 {\color{gray}(+1.30)} & 228.14 {\color{gray}(-9.48)} \\
    TaylorSeer \citep{taylorseer2025} ($N=3$) & 0.477 & 2.19$\times$ & 8.56 & 2.73 {\color{gray}(+0.38)} & 229.36 {\color{gray}(-8.26)} \\
    FoCa \citep{foca2025} ($N=3$) & 0.466 & 2.24$\times$ & 8.56 & 2.65 {\color{gray}(+0.30)} & 230.56 {\color{gray}(-7.06)} \\
    SpeCa \citep{speca2025} & 0.504 & 2.07$\times$ & \textbf{4.76} & 2.78 {\color{gray}(+0.43)} & 229.74 {\color{gray}(-7.88)} \\
    \rowcolor{gray!10} \textbf{TamedFlow} ($\tau=0.80$) & \textbf{0.464} & \textbf{2.25$\times$} & $6.83$ & \textbf{2.58} {\color{blue!60!black}(+0.23)} & \textbf{233.43} {\color{blue!60!black}(-4.19)} \\
    \bottomrule
  \end{tabular}}
\end{table}

\subsection{Experimental Setup}
\label{sec:setup}
\paragraph{Models and datasets.} The primary setting distills a 675M DiT-XL/2 teacher $\mathcal{T}_\phi$~\citep{peebles2023dit} into a 130M DiT-B/2 student $\mathcal{B}_\theta$ on class-conditional ImageNet at $256\times256$. We evaluate resolution scaling on ImageNet $512\times512$, text conditioning on MS-COCO, and video generation with Latte on FaceForensics. For COCO, a shared 19.9M text-conditioning adapter maps 77-token CLIP ViT-L/14 features into the DiT conditioning space; the teacher and student use this shared adapter while the ImageNet-trained VAH and LTE remain fixed. For Latte, the student, AlignNet, VAH, and LTE are retrained for the evaluated video-DiT setting. Complete training schedules, solver settings, and adaptation protocols are provided in Appendices~\ref{app:hyperparams} and~\ref{app:efficiency_accounting}.

\paragraph{Metrics and reporting.} On ImageNet, we report FID-50K and Inception Score (IS); on MS-COCO, FID-30K and CLIP-Score computed with OpenAI CLIP ViT-L/14; and on FaceForensics, FVD and video Inception Score. Image and text-to-image wall-clock latency and speedup are measured on a single A800-80GB GPU at batch size one and averaged over 10 runs. All baseline rows are measured in our local implementations under the same reporting protocol. Speedup is normalized to the corresponding teacher, and the LTE accept rate is the fraction of routing events that retain the student draft. FLOPs count backbone evaluations, whereas latency also includes VAH, LTE, rollback, and solver overhead.

\paragraph{Baselines.} We compare against schedule-reduction references (DDIM-25 and DDIM-20), full-scale few-step distillation methods (Progressive Distillation~\citep{salimans2022progressive}, LCM~\citep{lcm2023}, and ARD~\citep{kim2025ard}), and adaptive or speculative methods ($\Delta$-DiT~\citep{delta-dit2024}, TaylorSeer~\citep{taylorseer2025}, FoCa~\citep{foca2025}, FORA~\citep{fora2024}, SpeCa~\citep{speca2025}, and SpecDiff~\citep{specdiff2025}). Adaptive and speculative baselines are re-implemented in the matched DiT-XL/2 pipeline and measured locally under the same latency protocol. Few-step distillation baselines use the full teacher architecture without parameter compression. Appendix~\ref{app:adaptive_baselines} details the comparison scope and accounting assumptions.

\subsection{Main Results}
\label{sec:main_results}

\paragraph{Quality--latency operating points.} Table~\ref{tab:main} compares complete systems across speedup regimes. Four-step distillation baselines attain aggressive $14\times$--$15\times$ acceleration at the cost of severe FID degradation. TamedFlow instead targets moderate acceleration with tight quality retention: at $\tau=0.85$, it matches the DDIM-50 teacher within $0.02$ FID ($2.33$ vs. $2.35$) at $1.73\times$ speedup; at default $\tau=0.80$, it reaches FID $2.58$ at $2.25\times$ speedup ($0.464$~s).

In the $2.0\times$ to $2.5\times$ range under nearly matched latencies around $0.46$ to $0.48$ seconds, TamedFlow achieves better quality than FoCa at FID $2.65$ and TaylorSeer at FID $2.73$, alongside a higher Inception Score. This represents a quality advantage at matched latency rather than a throughput claim from millisecond-scale differences. Although SpeCa reports lower backbone FLOPs of $4.76$ T, its FID is $0.20$ worse, and backbone FLOPs omit runtime overheads like rollback, VAH, and LTE that are fully captured by measured latency. Thus, TamedFlow's edge lies in its empirical Pareto trade-off for the deployed system, not a FLOPs-only claim.

\paragraph{Qualitative study.} Figure~\ref{fig:qualitative} compares generated samples under matched speedup ($2.1\times$--$2.25\times$). TamedFlow preserves fine textures and object boundaries more cleanly across diverse classes. For example, it avoids the over-smoothed fur patterns on \textit{Tabby} and \textit{Boxer} seen in competing adaptive methods, and retains clear structural separation between the metal hoops and wood grain on \textit{Barrel}. These visual examples align with aggregate FID trends, while the following ablations isolate the underlying mechanisms.

\begin{figure}[t]
  \centering
  \includegraphics[width=1.0\linewidth]{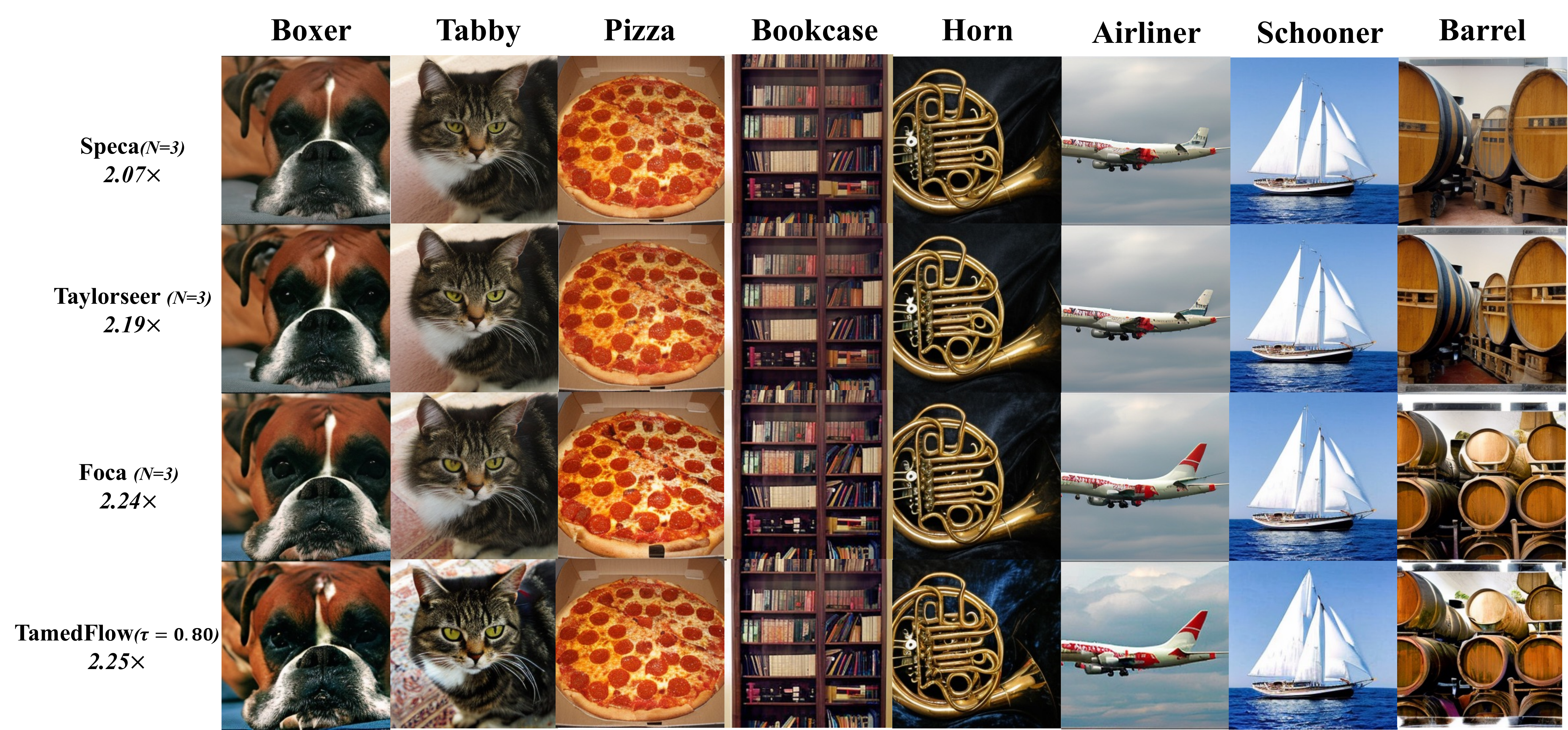}
  \caption{Qualitative comparison on ImageNet $256\times256$.}
  \label{fig:qualitative}
\end{figure}

\subsection{Component Ablation}
\label{sec:ablation}

Table~\ref{tab:abl} isolates the contribution of each architectural module and training objective at draft length $K=3$ and routing threshold $\tau=0.80$.

\begin{table}[h]
  \centering
  \small
  \caption{Ablation studies on ImageNet $256\times256$ at $K=3$ and $\tau=0.80$. Speedup is relative to the 50-step teacher baseline.}
  \label{tab:abl}
  \setlength{\tabcolsep}{4pt}
  \begin{tabular}{llcccc}
    \toprule
    \textbf{Study} & \textbf{Variant} & \textbf{Accept \%} $\uparrow$ & \textbf{Speedup} $\uparrow$ & \textbf{FID} $\downarrow$ & \textbf{IS} $\uparrow$ \\
    \midrule
    \multirow{4}{*}{Component} & Full TamedFlow & 63.6 & 2.25$\times$ & \textbf{2.58} & \textbf{233.43} \\
    & w/o AlignNet & 46.2 & 1.87$\times$ & 2.95 & 211.34 \\
    & w/o VAH & 39.3 & 1.63$\times$ & 3.23 & 205.71 \\
    & w/o LTE & -- & 5.80$\times$ & 5.56 & 195.07 \\
    \midrule
    \multirow{4}{*}{Training} & Full TamedFlow & 63.6 & 2.25$\times$ & \textbf{2.58} & \textbf{233.43} \\
    & w/o $\mathcal{L}_{\mathrm{roll}}$ & 54.2 & 2.02$\times$ & 2.71 & 230.15 \\
    & w/o $\mathcal{L}_{z}$ warmup & 46.2 & 1.81$\times$ & 3.29 & 203.50 \\
    & w/o $\mathcal{L}_{z}$ & 28.1 & 1.45$\times$ & 4.15 & 198.80 \\
    \bottomrule
  \end{tabular}
\end{table}

\textbf{Modules and training objectives.} AlignNet initializes student representations from intermediate teacher layers; removing it drops acceptance to 46.2\% and raises FID to 2.95 due to suboptimal representation grounding. Disabling the VAH reduces acceptance to 39.3\% and increases FID to 3.23, reflecting its role in absorbing local residual velocity mismatch (without claiming an explicit theoretical reduction in $\Delta J$). Without the LTE, unrouted student execution reaches 5.80$\times$ speedup but degrades severely to FID 5.56 in the absence of corrective recovery. For training losses, omitting rollout supervision $\mathcal{L}_{\mathrm{roll}}$ weakens intermediate velocity consistency and increases FID to 2.71. Removing the endpoint loss $\mathcal{L}_z$ has the most severe impact, dropping acceptance to 28.1\% and inflating FID to 4.15; its unrolled multi-step gradient directly penalizes compounding trajectory drift during autonomous rollouts. Removing the 10-epoch warmup on $\mathcal{L}_z$ yields FID 3.29 due to early optimization shocks before feature representations stabilize.

\paragraph{Routing-policy controls.}
Table~\ref{tab:routing_controls} compares policies using the same trained student, VAH, Euler grid, and guidance. At nearly identical latency (0.465 versus 0.464~s), LTE improves FID over random draft acceptance from 4.05 to 2.58. The random control enforces eight teacher updates per trajectory; LTE retains unconstrained $\tau=0.80$ routing. Thus this comparison supports learned routing at nearly matched latency, without establishing equal teacher usage. Static controls are faster but have worse FID. Appendix~\ref{app:routing_controls} gives the execution and call-accounting details.

\subsection{Routing Dynamics and Sensitivity}
\label{sec:hyperparam_dynamics}

\begin{figure}[t]
  \centering
  \includegraphics[width=0.8\linewidth]{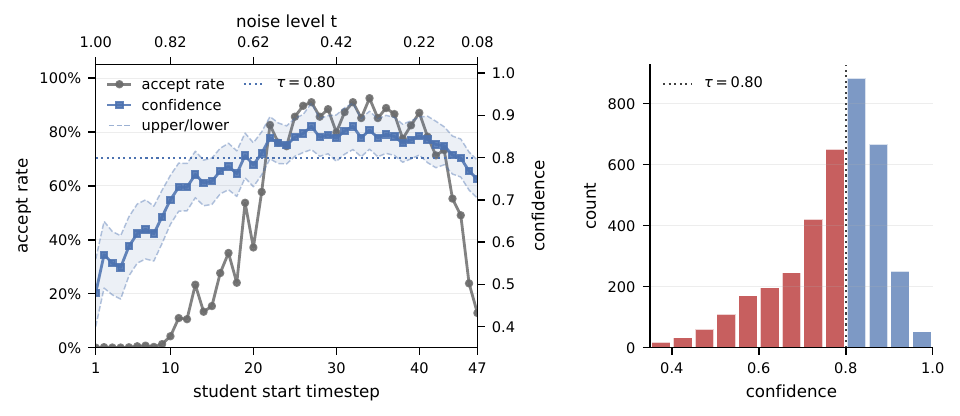}
  \caption{\textbf{LTE routing behavior across generation.} Left: confidence across sampling steps; the upper axis uses a noise-level coordinate, distinct from forward flow time. Right: score distributions of accepted and rejected drafts under $\tau=0.80$.}
  \label{fig:lte}
\end{figure}

Draft length $K$ and threshold $\tau$ balance speedup against fidelity (Appendix Figure~\ref{fig:hyper_k}). Speedup peaks at $K=3$ because longer drafts compound drift and increase rollback costs upon rejection. As shown in Figure~\ref{fig:lte}, LTE confidence remains high in the mid-noise regime but drops near the plotted noise level 0.22, invoking more teacher rescues late in sampling. The distribution shows clear operational separation between accepted and rejected drafts at the chosen operating point $\tau=0.80$.

In the reference trace at $K=3$ and $\tau=0.80$, 14 of 22 routing events are accepted, allowing student rollouts to cover $42$ of the $50$ integration steps. Detailed component profiling reveals that teacher rescue accounts for $71.9\%$ of wall-clock latency, whereas student rollouts take $14.9\%$, rejections with state rollback consume $12.7\%$, and LTE confidence scoring incurs only $0.5\%$. The lightweight estimator itself introduces negligible computational latency, but routing accuracy remains paramount: false acceptances allow subtle draft errors to compound and degrade fidelity, while false rejections trigger expensive teacher interventions that erode throughput. This concrete latency breakdown also explains why measured end-to-end speedup does not strictly match backbone FLOPs, as runtime costs from state rollback, synchronization, and auxiliary modules are omitted from nominal FLOP totals.

\subsection{Cross-Setting Evaluation}
\label{sec:coco-exp}

\begin{table}[h]
  \centering
  \small
  \caption{MS-COCO 30K $256\times256$ text-to-image generation using fixed-step DDIM. TaylorSeer, SpeCa, and FoCa are re-implemented within our DiT-XL/2 pipeline; $N$ denotes each method's acceleration setting rather than TamedFlow's draft horizon $K$.}
  \label{tab:coco}
  \setlength{\tabcolsep}{8pt}
  \renewcommand{\arraystretch}{1.1}
  \begin{tabular}{lcccc}
    \toprule
    \textbf{Method} & \textbf{Latency (s)} $\downarrow$ & \textbf{Speedup} $\uparrow$ & \textbf{FID} $\downarrow$ & \textbf{CLIP-S} $\uparrow$ \\
    \midrule
    \textit{Teacher DDIM-50} & \textit{1.184} & \textit{1.00$\times$} & \textit{13.54} & \textit{30.2} \\
    \textit{Teacher DDIM-25} & \textit{0.622} & \textit{1.90$\times$} & \textit{16.84} & \textit{28.9} \\
    \textbf{Pure Student} & 0.203 & 5.83$\times$ & 19.30 & 27.8 \\
    \midrule
    TaylorSeer ($N=3$) & 0.580 & 2.04$\times$ & 15.47 & 29.1 \\
    SpeCa & 0.561 & 2.11$\times$ & 16.63 & 28.4 \\
    FoCa ($N=3$) & 0.565 & 2.09$\times$ & 16.27 & 28.5 \\
    \rowcolor{gray!10}\textbf{TamedFlow} ($\tau=0.80$) & 0.603 & 1.96$\times$ & \textbf{14.86} & \textbf{29.6} \\
    \rowcolor{gray!10}\textbf{TamedFlow} ($\tau=0.75$) & \textbf{0.551} & \textbf{2.14$\times$} & 15.52 & 29.2 \\
    \bottomrule
  \end{tabular}
\end{table}

On MS-COCO, the teacher and student share a trained 19.9M text adapter while the ImageNet-trained VAH and LTE remain frozen. As reported in Table~\ref{tab:coco}, TamedFlow reaches FID $14.86$ and CLIP-Score $29.6$ at $1.96\times$ speedup at $\tau=0.80$, and $2.14\times$ speedup with FID $15.52$ at $\tau=0.75$, outperforming SpeCa and FoCa at comparable latency. This supports transfer of the frozen alignment modules under condition adaptation rather than unadapted zero-shot deployment. Tightening $\tau$ improves FID by $0.66$ and CLIP-Score by $0.4$ for $52$~ms latency, verifying that fidelity recovery stems from teacher intervention rather than student capacity.

We further evaluate spatial and temporal scaling. On ImageNet $512\times512$, expanding the latent grid fourfold under the same backbone, TamedFlow achieves FID $3.39$ at $2.41\times$ speedup in $1.594$~s, outperforming FoCa at FID $3.47$ with $2.36\times$ speedup as reported in Appendix~\ref{app:highres}. On the Latte video DiT retrained on FaceForensics at $256\times256$ resolution across 16 frames in Table~\ref{tab:latte}, TamedFlow attains FVD $225.10$ and IS $68.90$ at $2.22\times$ speedup, surpassing FoCa at FVD $248.30$ and TaylorSeer at FVD $253.10$. These results confirm that the draft and rescue mechanism transfers reliably across resolutions, conditions, and video modalities.

\begin{table}[h]
  \centering
  \small
  \caption{Latte video-DiT transfer on FaceForensics ($256\times256$, 16 frames).}
  \label{tab:latte}
  \setlength{\tabcolsep}{9pt}
  \renewcommand{\arraystretch}{1.1}
  \begin{tabular}{lccc}
    \toprule
    \textbf{Method} & \textbf{Speedup} & \textbf{FVD} $\downarrow$ & \textbf{IS} $\uparrow$ \\
    \midrule
    Teacher (Latte-XL) & 1.00$\times$ & 208.50 & 71.20 \\
    FoCa & 2.06$\times$ & 248.30 & 66.80 \\
    TaylorSeer & 2.10$\times$ & 253.10 & 65.90 \\
    \rowcolor{gray!10}\textbf{TamedFlow} & \textbf{2.22$\times$} & \textbf{225.10} & \textbf{68.90} \\
    \bottomrule
  \end{tabular}
\end{table}

\paragraph{Deployment cost.} TamedFlow requires offline training and maintains the teacher, student, VAH, and LTE co-resident in memory, increasing peak allocated VRAM from $2.570$ to $3.075$~GB ($\sim$0.5~GB overhead). This footprint reflects the deliberate choice to retain a high-capacity teacher for selective rescue rather than discarding it. Headline latency is benchmarked for batch-one latent integration without VAE decoding, characterizing interactive generation rather than batched serving throughput. Teacher offloading and weight quantization represent orthogonal opportunities to mitigate co-residency overhead in production.

\section{Conclusion}
\label{sec:conclusion}

In this work, we introduced TamedFlow, an acceleration framework that reconciles parameter compression with adaptive execution for Diffusion Transformers. By training a compact student for short-horizon trajectory agreement and employing asymmetric routing with selective teacher rescue, TamedFlow achieves FID $2.58$ at $2.25\times$ speedup over DDIM-50 on ImageNet $256\times256$, with transfer evaluated across resolutions, text conditioning, and video. Our findings show that pointwise alignment can leave compounding trajectory divergence during autonomous drafting, creating the Combination Gap. Controlled ablations support rollout supervision, and routing controls favor LTE over random decisions at nearly matched latency, suggesting that compression and adaptive inference can be made compatible when the compact model is trained as a multi-step drafter and the sampler retains a recovery path back to the teacher.

\bibliographystyle{tamedflow_preprint}
\bibliography{references}
\clearpage

\appendix
\numberwithin{equation}{section}

\setcounter{section}{0}
\setcounter{subsection}{0}

\renewcommand{\thesection}{\Alph{section}}
\renewcommand{\thesubsection}{\thesection.\arabic{subsection}}

\titleformat{\section}
  {\normalfont\Large\bfseries}
  {Appendix~\thesection}
  {1em}
  {}

\titleformat{\subsection}
  {\normalfont\large\bfseries}
  {\thesubsection}
  {1em}
  {}

\section{Training Hyperparameters}
\label{app:hyperparams}

Table~\ref{tab:hyperparams} summarizes the training hyperparameters for the two stages of TamedFlow.

\begin{table}[h]
  \centering
  \caption{Training hyperparameters for TamedFlow.}
  \label{tab:hyperparams}
  \small
  \setlength{\tabcolsep}{5pt}
  \begin{tabular}{lcc}
    \toprule
    \textbf{Component} & \textbf{Stage 1} & \textbf{Stage 2} \\
    \midrule
    \textbf{Optimizer} & \multicolumn{2}{c}{AdamW ($\beta_1=0.9$, $\beta_2=0.999$)} \\
    \textbf{Base learning rate} & $5\times10^{-5}$ & $1\times10^{-4}$ \\
    \textbf{Weight decay} & \multicolumn{2}{c}{0.0} \\
    \textbf{Learning rate schedule} & \multicolumn{2}{c}{cosine decay with warmup} \\
    \textbf{Warmup epochs} & \multicolumn{2}{c}{5} \\
    \textbf{Max epochs} & 10 & 80 \\
    \textbf{Batch size (per GPU)} & \multicolumn{2}{c}{32} \\
    \textbf{Global batch size} & \multicolumn{2}{c}{256} \\
    \textbf{Gradient clipping} & \multicolumn{2}{c}{1.0} \\
    \midrule
    \textbf{Loss terms} & \multicolumn{2}{c}{} \\
    \quad Velocity loss $\mathcal{L}_v$ & \checkmark & \checkmark \\
    \quad Alignment loss $\mathcal{L}_{\mathrm{align}}$ & $\lambda=0.25$ & -- \\
    \quad Rollout loss $\mathcal{L}_{\mathrm{roll}}(K)$ & -- & $\lambda=0.8$ \\
    \quad Endpoint loss $\mathcal{L}_z$ & -- & $\lambda_z$ (10-epoch warmup to 30.0) \\
    \quad Confidence loss $\mathcal{L}_{\mathrm{conf}}$ & -- & $\lambda=1.0$ \\
    \midrule
    \textbf{Model details} & \multicolumn{2}{c}{} \\
    \quad Student architecture & \multicolumn{2}{c}{DiT-B/2 (12 layers)} \\
    \quad Student parameters & \multicolumn{2}{c}{130M} \\
    \quad Teacher architecture & \multicolumn{2}{c}{DiT-XL/2 (28 layers)} \\
    \quad Teacher parameters & \multicolumn{2}{c}{675M} \\
    \quad AlignNet architecture & \multicolumn{2}{c}{3 independent MLPs, $768\!\to\!1152\!\to\!1152$} \\
    \quad VAH architecture & \multicolumn{2}{c}{3-layer Conv2D residual head, hidden=256} \\
    \quad LTE architecture & \multicolumn{2}{c}{3 projections; MLP $1024\!\to\!1024\!\to\!256\!\to\!1$} \\
    \quad LTE parameters & \multicolumn{2}{c}{\textasciitilde4M (3\% of student)} \\
    \midrule
    \textbf{Inference hyperparameters} & \multicolumn{2}{c}{} \\
    \quad Draft length $K$ & \multicolumn{2}{c}{3 (training: $K\sim\mathcal{U}(1,5)$)} \\
    \quad Temperature $T_z$ & \multicolumn{2}{c}{0.5} \\
    \quad Temperature $T_v$ & \multicolumn{2}{c}{0.1} \\
    \quad Mixing weight $\alpha$ & \multicolumn{2}{c}{0.7} \\
    \quad Default $\tau$ & \multicolumn{2}{c}{0.8} \\
    \midrule
    \textbf{LTE training} & \multicolumn{2}{c}{} \\
    \quad Gradient isolation & -- & \checkmark (stop-gradient) \\
    \quad LTE warmup epochs & \multicolumn{2}{c}{5} \\
    \quad LTE learning rate & \multicolumn{2}{c}{$1\times10^{-4}$} \\
    \bottomrule
  \end{tabular}
\end{table}

\section{Baseline Comparison Scope}
\label{app:adaptive_baselines}

Table~\ref{tab:main} compares end-to-end operating points, not isolated routers. The adaptive baselines retain the 675M teacher, whereas TamedFlow combines a 130M draft student with selective teacher rescue; compression, training, solver choice, routing, and rollback therefore all contribute to the measured result. All baseline entries, including SpecDiff, are evaluated in our local implementations under the stated protocol. TamedFlow FLOPs are backbone-call estimates; end-to-end latency additionally includes VAH, LTE, solver, and rollback overhead.

Table~\ref{tab:same_student_validation} serves a different purpose. It is a quality diagnostic in which adaptive reuse is attached directly to the compressed student, and therefore supports the motivating Combination Gap rather than a latency-matched routing comparison.

\section{Higher-Resolution Evaluation}
\label{app:highres}

We evaluate spatial scaling on ImageNet at $512\times512$ with DiT-XL/2, DDIM-50, and one A800 GPU. The comparison holds the resolution and protocol fixed and tests whether the primary quality--latency ordering survives the larger latent grid.

\begin{table}[h]
  \centering
  \small
  \caption{ImageNet $512\times512$ evaluation with DiT-XL/2 using fixed-step DDIM.}
  \label{tab:highres}
  \setlength{\tabcolsep}{8pt}
  \begin{tabular}{lcccc}
    \toprule
    \textbf{Method} & \textbf{Latency (s)} $\downarrow$ & \textbf{Speedup} $\uparrow$ & \textbf{FID} $\downarrow$ & \textbf{IS} $\uparrow$ \\
    \midrule
    Teacher (DiT-XL/2, 675M) & 3.842 & 1.00$\times$ & 2.87 & 235.10 \\
    TaylorSeer ($N=3$) & 1.715 & 2.24$\times$ & 3.76 & 226.75 \\
    FoCa ($N=3$) & 1.628 & 2.36$\times$ & 3.47 & 228.10 \\
    \textbf{TamedFlow ($\tau=0.80$)} & \textbf{1.594} & \textbf{2.41$\times$} & \textbf{3.39} & \textbf{230.85} \\
    \bottomrule
  \end{tabular}
\end{table}

\section{Multi-Seed Evaluation}
\label{app:multiseed}

We generate 50K ImageNet samples for each of three evaluation seeds using fixed checkpoints. These repeats quantify sampling variability, not independent training-run variability; every method follows the same evaluation protocol.

\begin{table}[h]
  \centering
  \small
  \caption{Three-seed ImageNet evaluation (mean $\pm$ standard deviation over seeds).}
  \label{tab:multiseed}
  \setlength{\tabcolsep}{9pt}
  \begin{tabular}{lccc}
    \toprule
    \textbf{Method} & \textbf{Speedup} & \textbf{FID} & \textbf{IS} \\
    \midrule
    DDIM-50 Teacher & 1.00$\times$ & $2.36\pm0.02$ & $237.24\pm0.48$ \\
    TaylorSeer ($N=3$) & 2.19$\times$ & $2.71\pm0.04$ & $229.13\pm0.62$ \\
    FoCa ($N=3$) & 2.21$\times$ & $2.64\pm0.03$ & $231.24\pm0.58$ \\
    \textbf{TamedFlow, $\tau=0.80$} & \textbf{2.26$\times$} & \textbf{$2.56\pm0.03$} & \textbf{$233.89\pm0.51$} \\
    \bottomrule
  \end{tabular}
\end{table}

\section{Efficiency and Transfer Accounting}
\label{app:efficiency_accounting}

\paragraph{Deployment footprint.} In a controlled A800, batch-size-2, FP32, 50-step latent-only profile with VAE decoding excluded, peak allocated VRAM is 2.570 GB for the teacher, 2.580 GB for FoCa ($N=3$), 3.066 GB for TaylorSeer ($N=3,O=3$), and 3.075 GB for TamedFlow. Co-deploying the student and LTE therefore adds about 0.5 GB relative to the teacher-only profile.
For the ImageNet and MS-COCO image experiments, all methods use classifier-free guidance at scale $3.5$. The teacher FLOP reference includes the complete conditional and unconditional CFG evaluation, while the CFG-distilled student reference counts one student forward; this convention defines the teacher/student references and the TamedFlow backbone estimates.

\paragraph{Routing accounting.} In the reference trace at $\tau=0.80$, 14 of 22 routing events are accepted and 8 are rejected, giving a 63.6\% event acceptance rate. Since each accepted event advances $K=3$ Student steps, accepted drafts cover $42/50=84\%$ of the nominal trajectory; the remaining 8 steps use Teacher rescue. The separately measured component profile records 0.07075 s for accepted-draft buckets (14.9\%), 0.06031 s for rejected-draft buckets including rollback (12.7\%), 0.34142 s for Teacher rescue (71.9\%), and 0.00237 s for LTE scoring (0.5\%). These buckets sum to 0.47485 s (up to rounding). They are event-level timing buckets from the profiling run, so dividing a bucket by the number of backbone calls would incorrectly attribute rollback and routing overhead to an individual forward. The component profile and the headline latency measurement were collected in different runs: the 0.464 s headline is the separately measured ten-run average. We therefore report both traces without using one to reconstruct the other.

\paragraph{Backbone FLOP accounting.} The 50-step references imply $f_S=2.30/50=0.046$~T per CFG-distilled student forward and $f_T=23.74/50=0.4748$~T per teacher update including both CFG branches. With the executed call counts defined in Appendix~\ref{app:routing_controls},
\begin{equation}
F_{\mathrm{backbone}}=N_Sf_S+N_Tf_T.
\end{equation}
For the reference execution with 14 accepted and eight rejected full-length drafts, $N_S=66$, $N_T=8$, and $N_{\mathrm{LTE}}=22$, yielding $66(0.046)+8(0.4748)=6.8344$~T. The full-length 13-accept/11-reject reference gives $72(0.046)+11(0.4748)=8.5348$~T. The rounded entries in Table~\ref{tab:main} are call-based estimates for these reference configurations, not evaluation-set means. In contrast, a static eight-teacher policy has 42 student forwards and a backbone subtotal of $42(0.046)+8(0.4748)=5.7304$~T under the same accounting. VAH, LTE, rollback, and solver arithmetic are excluded from these backbone subtotals and included in measured latency. The 22 LTE calls do not require 22 additional student forwards because their features are reused from drafting. Logical conditional and unconditional teacher evaluations may be batched; eight guided teacher updates account for both branches, whereas the CFG-distilled student uses one branch per update. Backbone FLOPs alone do not predict measured latency.

\paragraph{Latte video protocol.} We retrain the student, AlignNet, VAH, and LTE for the evaluated Latte video-DiT setting while keeping the TamedFlow training and inference configuration consistent with the primary experiments: Stage~1 for 10 epochs, Stage~2 for 80 epochs, $K=3$ at inference, randomized $K\sim\mathcal{U}(1,5)$ during rollout training, and the same loss weights and threshold as ImageNet. The original Latte teacher/student block configurations are retained; no width, depth, or temporal-factorization search is performed. Spatial and temporal tokens are treated as one spatiotemporal token sequence for AlignNet matching and LTE mean pooling, so the LTE adds no separate temporal backbone. Evaluation uses FaceForensics clips at $256\times256$ resolution with 16 frames. We use the same clips and evaluation protocol for all methods, batch size 1, deterministic 50-step sampling, and latency measured end-to-end on 8 A800-80GB GPUs. VAE decoding is excluded from the reported latency. The video table reports FVD and video Inception Score together with the teacher-normalized speedup.

\paragraph{MS-COCO adaptation protocol.} The shared 19.9M text-conditioning adapter is trained jointly with the teacher and student for 20 epochs (5 Stage~1 and 15 Stage~2). The implementation uses AdamW with learning rate $10^{-4}$ for the text adapter, a $0.1$ learning-rate multiplier for the backbone parameters, weight decay $0.01$, text-dropout probability $0.05$, gradient clipping at $1.0$, and EMA decay $0.9999$. The DiT backbones are frozen except for the text-conditioning parameters; the LTE and VAH are frozen during this adapter-only adaptation. The adapter maps 77-token CLIP ViT-L/14 features into the DiT conditioning width and is shared by teacher and student.

\begin{figure}[!htbp]
  \centering
  \includegraphics[width=0.78\linewidth]{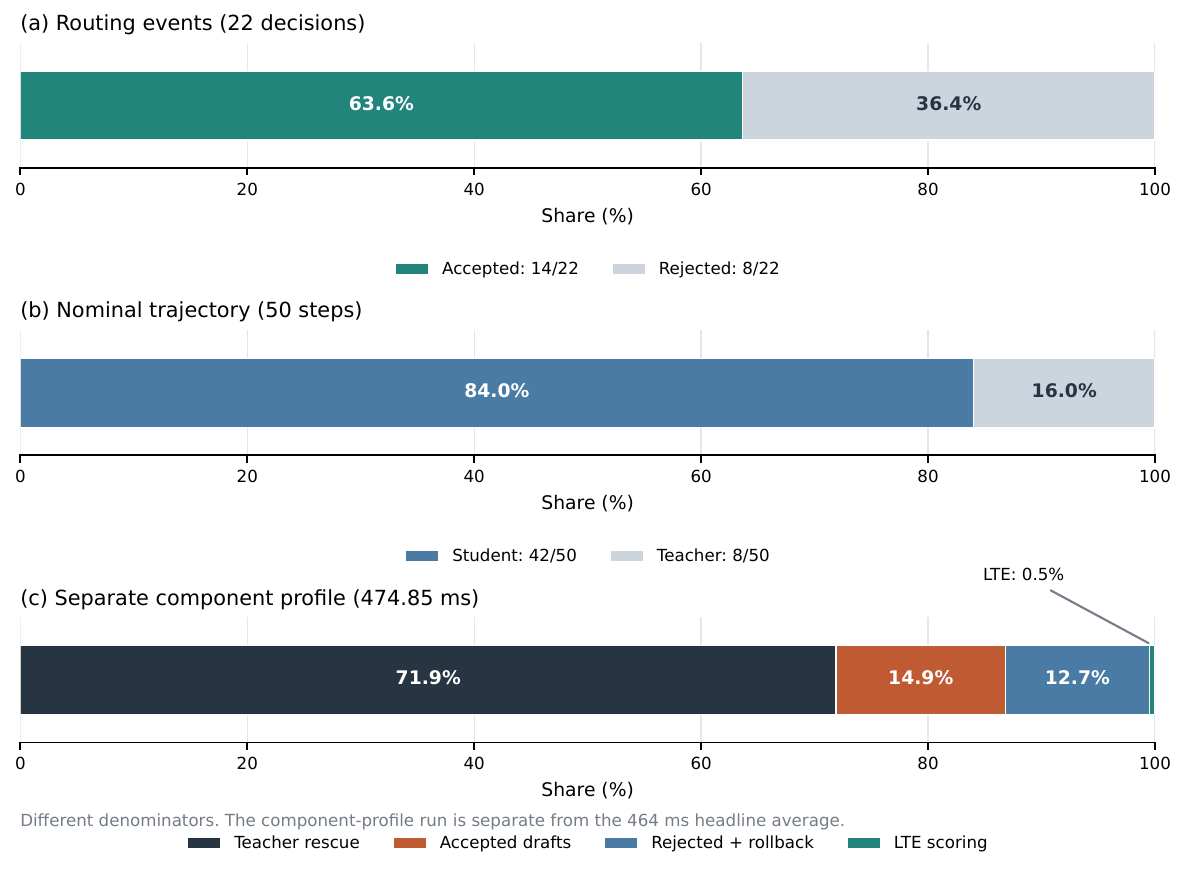}
  \caption{Compact routing and execution accounting at $K=3$, $\tau=0.80$. The panels use distinct denominators: (a) 14 accepted and 8 rejected routing events; (b) 42 student-covered and 8 teacher-rescue nominal steps; and (c) the separately collected component profile. The component profile totals $474.85$\,ms and is not a decomposition of the separate $464$\,ms headline average.}
  \label{fig:routing_profile}
\end{figure}

\begin{figure*}[!htbp]
  \centering
  \includegraphics[width=0.98\textwidth]{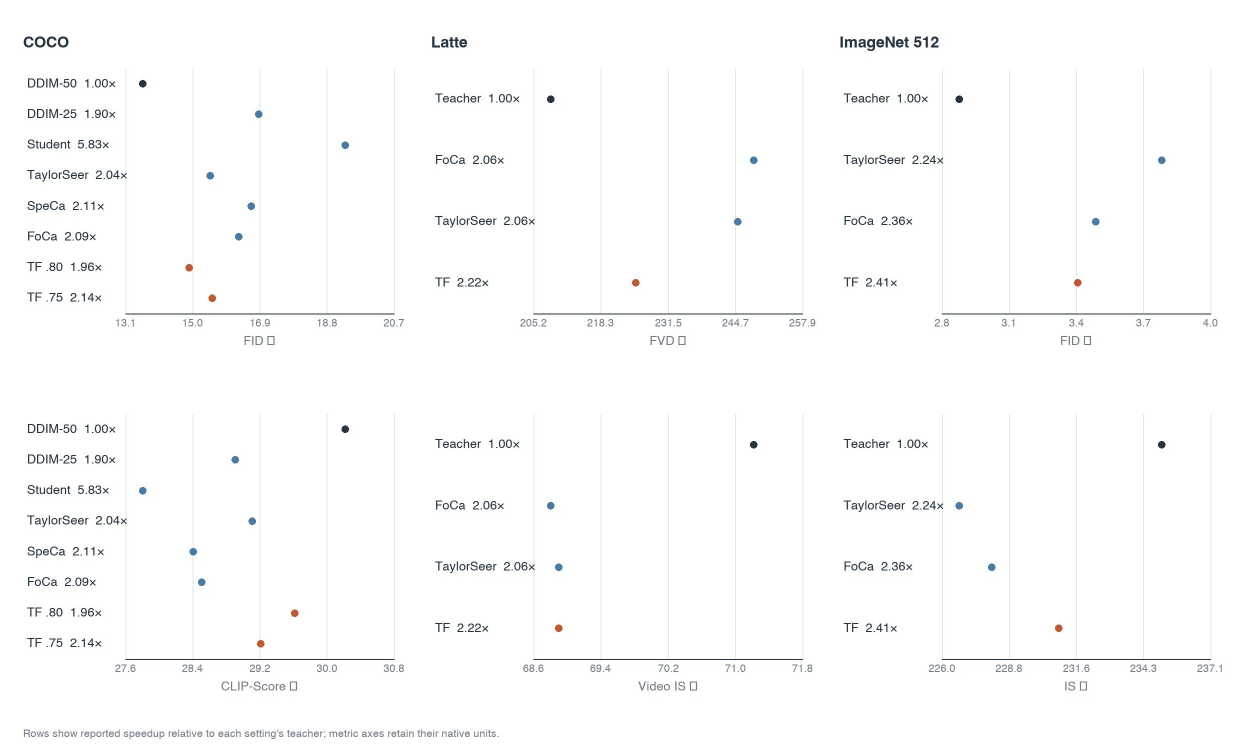}
  \caption{Cross-setting transfer shown as six compact native-unit panels. Columns are COCO, Latte/FaceForensics, and ImageNet $512\times512$; the top row reports quality metrics and the bottom row reports semantic or recognition scores. Row labels include the reported speedup relative to the teacher for that setting.}
  \label{fig:transfer_profile}
\end{figure*}

\section{Proof of the K-Step Trajectory Alignment Bound}
\label{app:proof_prop1}
Let $x_S^{(k)}$ and $x_T^{(k)}$ denote the states of the student and teacher at step $k$ during a $K$-step sequential Euler rollout, starting from the identical initial state $x_S^{(0)} = x_T^{(0)} = x_t$. We assume a fixed ODE solver step size $h$. With time arguments suppressed at the common $t_k=t+kh$, define the baseline velocity mismatch evaluated strictly on the teacher's trajectory as $\delta_v^{(k)} \triangleq \tilde v_S(x_T^{(k)}) - v_T(x_T^{(k)})$, with its supremum bounded as $\|\delta_v\|_\infty \triangleq \max_k \|\delta_v^{(k)}\|_2$. Define the teacher's Lipschitz constant $L_T \triangleq \sup_{x,t}\|J_{v_T}(x,t)\|_2$, and the Jacobian mismatch $\Delta J \triangleq \sup_{x,t}\|J_{\tilde v_S}(x,t)-J_{v_T}(x,t)\|_2$.
\begin{proof}
\textbf{Setup: recursive error accumulation.}
In a multi-step rollout, each model feeds its own prediction into the next Euler step:
\begin{equation}
x_S^{(k+1)} = x_S^{(k)} + h\,\tilde v_S(x_S^{(k)}),\qquad
x_T^{(k+1)} = x_T^{(k)} + h\,v_T(x_T^{(k)}).
\end{equation}
Let $e_k \triangleq x_S^{(k)} - x_T^{(k)}$ denote the spatial deviation at step $k$, with the initial condition $e_0 = 0$. Subtracting the two updates yields the discrete error dynamics:
\begin{equation}
e_{k+1} = e_k + h\,\big(\tilde v_S(x_S^{(k)}) - v_T(x_T^{(k)})\big).
\end{equation}
\textbf{Decomposing the velocity difference via integral Taylor expansion.}
To separate the pointwise baseline error from the accumulated spatial drift, we add and subtract the student's velocity at the teacher state, $\tilde v_S(x_T^{(k)})$:
\begin{equation}
\tilde v_S(x_S^{(k)}) - v_T(x_T^{(k)}) =
\underbrace{\big(\tilde v_S(x_T^{(k)}) - v_T(x_T^{(k)})\big)}_{\delta_v^{(k)}}
+
\underbrace{\big(\tilde v_S(x_S^{(k)}) - \tilde v_S(x_T^{(k)})\big)}_{\text{Drift induced by } e_k}.
\end{equation}
Since the velocity field $\tilde v_S: \mathbb{R}^D \to \mathbb{R}^D$ is vector-valued, the scalar Mean Value Theorem does not apply. We instead write the spatial drift using the first-order Taylor expansion with an integral remainder along the line segment connecting $x_T^{(k)}$ and $x_S^{(k)}$:
\begin{equation}
\tilde v_S(x_S^{(k)}) - \tilde v_S(x_T^{(k)})
=
\left(
\int_0^1 J_{\tilde v_S}\big(x_T^{(k)} + \tau e_k\big) \,\mathrm{d}\tau
\right)e_k.
\end{equation}
Taking the $L_2$ norm of the error recurrence, substituting the decomposition and the integral expansion, and applying the triangle inequality yields:
\begin{equation}
\begin{aligned}
\|e_{k+1}\|_2
&\leq
\|e_k\|_2
+ h\,\|\delta_v\|_\infty
+ h\,
\bigg\|
\int_0^1 J_{\tilde v_S}\big(x_T^{(k)} + \tau e_k\big) \,\mathrm{d}\tau
\bigg\|_2
\|e_k\|_2 \\
&\leq
\|e_k\|_2
+ h\,\|\delta_v\|_\infty
+ h
\left(
\int_0^1 \sup_x\|J_{\tilde v_S}(x)\|_2 \,\mathrm{d}\tau
\right)
\|e_k\|_2.
\end{aligned}
\end{equation}
Here, $\sup_x\|J_{\tilde v_S}(x)\|_2$ is the student's Lipschitz constant $L_S$. By splitting the student's Jacobian into the teacher term and the mismatch term, we bound it with the matrix triangle inequality:
\begin{equation}
L_S
=
\sup_{x,t}\|J_{\tilde v_S}\|_2
=
\sup_{x,t}\|J_{v_T} + (J_{\tilde v_S} - J_{v_T})\|_2
\leq L_T + \Delta J.
\end{equation}
Since the integral over $[0,1]$ equals one, substituting the Jacobian bound into the preceding inequality gives the recurrence:
\begin{equation}
\|e_{k+1}\|_2
\leq
\big(1 + h(L_T + \Delta J)\big)\|e_k\|_2
+
h\,\|\delta_v\|_\infty.
\end{equation}
\textbf{Solving the recurrence.}
Let $C \triangleq L_T + \Delta J$. Unrolling the recurrence with the initial condition $e_0 = 0$ yields a geometric series:
\begin{equation}
\|e_K\|_2
\leq
h\,\|\delta_v\|_\infty
\sum_{i=0}^{K-1} (1 + hC)^i
=
\frac{\|\delta_v\|_\infty}{C}
\big((1 + hC)^K - 1\big).
\end{equation}
To expose the polynomial dependence on $K$ in the few-step regime, we use the Taylor expansion of the exact scalar term $(1+hC)^K$:
\begin{equation}
(1 + hC)^K
=
1 + K(hC)
+ \frac{K(K-1)}{2}(hC)^2
+ \mathcal{O}\big(K^3(hC)^3\big).
\end{equation}
For $C=0$ the recurrence gives $\|e_K\|_2\leq Kh\|\delta_v\|_\infty$. For $KhC\ll1$, expansion of the geometric upper bound gives:
\begin{equation}
\|e_K\|_2
\leq
K h\|\delta_v\|_\infty
+\frac{K(K-1)}{2}h^2C\|\delta_v\|_\infty
+\mathcal{O}\!\left(K^3h^3C^2\|\delta_v\|_\infty\right).
\end{equation}
\textbf{Interpretation.}
The exact bound isolates the baseline velocity mismatch and shows how the combined Lipschitz/mismatch constant $C=L_T+\Delta J$ amplifies errors with the rollout horizon. In the few-step expansion, this produces the stated quadratic-in-$K$ sensitivity. Pointwise velocity matching directly supervises the forcing term but does not explicitly constrain the spatial amplification constant; it may nevertheless affect both. The expansion concerns an upper bound, not a measured error growth law.
\end{proof}

\section{Scope of the VAH Mechanism}
\label{app:proof_vah}
VAH is a residual correction trained through the Stage-2 velocity and trajectory objectives in Eq.~\eqref{eq:stage2}. No separate explicit Jacobian-matching objective is specified by this training recipe. For $v_S^{\mathrm{aligned}}=v_S+\mathrm{VAH}(v_S)$, the chain rule gives $J_{v_S^{\mathrm{aligned}}}=(I+J_{\mathrm{VAH}})J_{v_S}$. This permits changes to local derivatives but does not guarantee reduced Jacobian mismatch. Figure~\ref{fig:bound_diagnostics} reports velocity mismatch and endpoint drift only. It provides no Jacobian diagnostic: reduction of $\Delta J$ remains a theoretical possibility, not an experimentally established mechanism.

\section{Inference Details}
\label{app:inference_details}

\paragraph{Velocity-field interface and guidance.} The symbol $v_T$ denotes the effective teacher right-hand side used by the Euler solver, including the prescribed conditioning and guidance, rather than specifying the raw prediction type of the underlying DiT. Student--teacher velocity losses and teacher reference rollouts are expressed in this same solver coordinate system. The single-branch student is distilled toward the guided teacher field; a teacher rescue still evaluates the complete guided teacher update. The linear-interpolant objective in Section~\ref{sec:preliminaries} introduces the velocity-form convention and, by itself, does not identify the pretrained teacher's native output parameterization or noise schedule. An architecture name alone does not determine a checkpoint or an output-to-ODE conversion.

\paragraph{Time coordinates.} The solver and all rollout equations use increasing flow time $t$, with noise at $t=0$ and data at $t=1$. Figure~\ref{fig:lte} retains a decreasing noise-level plotting coordinate on its upper axis; that axis label is not the forward-time variable in the solver equations. The reported value 0.22 describes that plotted noise coordinate, not forward flow time. Sampling-step indices are used to compare the routing profile with execution order.

Algorithm~\ref{alg:inference} gives the full routing loop under the forward flow-time convention $t_0<\cdots<t_N$. The student drafts up to $K$ steps without teacher access, and the LTE scores the draft from the last student forward, the same type of input used during LTE training. Rejection restores the pre-draft state before one teacher rescue step; no endpoint backbone call is added for scoring.

\begin{algorithm}[t]
\caption{TamedFlow Inference}
\label{alg:inference}
\begin{algorithmic}[1]
\State \textbf{Input:} $x \sim \mathcal{N}(0,I)$, $y$, discrete times $\{t_0, \dots, t_N\}$, $K=3$, threshold $\tau$
\State $i \gets 0$
\While{$i < N$}
  \State $m\gets\min(K,N-i)$ \Comment{Clip the terminal draft horizon}
  \State $x_{\mathrm{start}} \gets x$
  \For{$k = 1$ \textbf{to} $m$} \Comment{Student drafts $m$ steps}
    \State $(v_S, \mathrm{tokens}_k) \gets \mathcal{B}_\theta(x, t_{i+k-1}, y)$
    \State $v_S\gets v_S+\mathrm{VAH}(v_S)$
    \State $x \gets x + (t_{i+k} - t_{i+k-1}) \cdot v_S$
  \EndFor
  \State $c \gets \mathrm{LTE}(\overline{\mathrm{tokens}}_m,e_{t_{i+m-1}},e_y)$ \Comment{Last draft features; no teacher access}
  \If{$c \geq \tau$}
    \State $i \gets i + m$ \Comment{Accept: advance by $m$, no teacher call}
  \Else
    \State $x \gets x_{\mathrm{start}}$ \Comment{Reject: roll back to pre-draft state}
    \State $v_T \gets \mathcal{T}_\phi(x, t_i, y)$ \Comment{Teacher rescue evaluation step}
    \State $x \gets x + (t_{i+1} - t_i) \cdot v_T$
    \State $i \gets i + 1$
  \EndIf
\EndWhile
\State \Return $x$
\end{algorithmic}
\end{algorithm}

\section{LTE Calibration and Routing Metrics}
\label{app:lte_metrics}

Table~\ref{tab:lte_metrics} reports held-out routing diagnostics at $\tau=0.80$ and $K=3$. The LTE regresses the continuous target $c^\star$ in Eq.~\eqref{eq:c_star}; binary labels are introduced only to summarize thresholded routing errors.

\paragraph{Reference labels and routing errors.} A draft event is labeled offline as accept when $c^\star\geq0.80$ and reject otherwise; the online decision accepts when the prediction $c\geq0.80$. FAR uses the 4,683 reference-reject events as its denominator, FRR uses the 11,984 reference-accept events, and AUC ranks events by continuous $c$. These labels measure agreement with the specified teacher-consistency target, not preservation of final sample quality.

\paragraph{Validation protocol.} Threshold and score-temperature choices are selected on a calibration split and then frozen. Table~\ref{tab:lte_metrics} is computed on a held-out trajectory-level split, preventing draft events from one trajectory from appearing in both sets. Continuous-target ECE is the event-weighted 10-bin discrepancy between mean $c$ and mean $c^\star$; it measures score agreement rather than a probability of correct generation.

\begin{table}[h]
  \centering
  \caption{LTE classification and calibration metrics on the ImageNet $256\times256$ validation set at $\tau=0.80$, $K=3$. Support counts reference labels obtained by thresholding $c^\star$ at 0.80. FAR and FRR use reference-reject and reference-accept events as their respective denominators. AUC uses continuous predictions; continuous-target ECE uses 10 confidence bins.}
  \label{tab:lte_metrics}
  \begin{tabular}{l|c|c}
    \toprule
    \textbf{Support (Reference)} & \textbf{Accept} & \textbf{Reject} \\
    \midrule
    Count & 11,984 & 4,683 \\
    \midrule
    \textbf{Metric} & \multicolumn{2}{c}{\textbf{Value}} \\
    \midrule
    False Accept Rate (FAR) & \multicolumn{2}{c}{0.063} \\
    False Reject Rate (FRR) & \multicolumn{2}{c}{0.051} \\
    AUC & \multicolumn{2}{c}{\textbf{0.94}} \\
    Continuous-target ECE (10 bins) & \multicolumn{2}{c}{0.034} \\
    \bottomrule
  \end{tabular}
\end{table}

The reference-positive fraction is $11{,}984/16{,}667=71.9\%$, whereas the online predicted acceptance rate is 63.6\%. These use different labels and event populations, so their difference alone does not identify distribution shift. Held-out diagnostics do not establish error rates along free-running online trajectories. Thresholding predicted confidence separates routed events by construction; these held-out diagnostics do not establish calibration against actual draft error.

\subsection{Routing-control protocol and call accounting}
\label{app:routing_controls}

\begin{table}[t]
  \centering
  \small
  \caption{Routing controls on ImageNet $256\times256$. Latency and FID-50K are measured. Teacher updates are fixed by the control policy or adaptive for LTE; this is not an equal-teacher-budget comparison with LTE.}
  \label{tab:routing_controls}
  \setlength{\tabcolsep}{5pt}
  \begin{tabular}{lccc}
    \toprule
    \textbf{Routing policy} & \textbf{Teacher updates} & \textbf{Latency (s)} $\downarrow$ & \textbf{FID} $\downarrow$ \\
    \midrule
    Pure Student & 0 & 0.180 & 5.56 \\
    Static: Random Steps & 8 & 0.348 & 3.72 \\
    Static: Fixed Periodic & 8 & 0.335 & 3.45 \\
    Dynamic: Random Accept & 8 & 0.465 & 4.05 \\
    \textbf{Dynamic: LTE} & Adaptive & 0.464 & \textbf{2.58} \\
    \bottomrule
  \end{tabular}
\end{table}

The controls in Table~\ref{tab:routing_controls} use the same trained student checkpoint, VAH, 50-step Euler grid, CFG configuration, and ImageNet FID-50K protocol. Latency follows the single-A800, batch-one protocol in the main experiments. These measurements evaluate inference policies without retraining the student.

\paragraph{Static policies.} Pure Student performs all 50 updates with the student and never calls the teacher. Random Steps independently selects a set of eight distinct integration steps per trajectory before sampling and replaces those student updates with teacher updates. Fixed Periodic uses the one-based integration-step indices $6,12,18,24,30,36,42,48$. Neither static policy executes a speculative student draft at a teacher-designated step. Both therefore execute exactly 42 student forwards and eight teacher updates.

\paragraph{Budget-constrained random decisions.} Random Accept uses the same $K=3$ draft-and-rollback structure as LTE, but substitutes a random accept decision with nominal probability $p=0.636$. The policy enforces exactly eight teacher updates per trajectory. Once that budget is exhausted, subsequent drafts are accepted. A budget-feasibility override forces rejection when accepting a draft would leave too few integration steps to spend the remaining teacher budget. A rejection discards the draft, restores its starting state, and advances one teacher step. Consequently, the actual decisions are not independent Bernoulli draws throughout the trajectory, and $p$ is not a measured acceptance rate.

\paragraph{LTE policy.} LTE uses the original confidence threshold $\tau=0.80$ without a fixed per-trajectory teacher budget. Its teacher and student call counts depend on the trajectory. The previously reported 14 accepted and eight rejected events describe a reference execution, not a measured evaluation-set mean. Thus the random control and LTE share the execution framework and nearly match measured latency, but are not established to have identical average teacher usage.

\paragraph{LTE evaluations versus backbone calls.} Let $A$ and $R$ denote accepted and rejected draft events. Each draft triggers exactly one LTE evaluation, regardless of whether it is accepted; only rejection triggers a teacher update. For the terminal-clipped $K=3$ loop,
\begin{equation}
N_T=R,\qquad N_{\mathrm{LTE}}=A+R,\qquad
N_S=\sum_{j=1}^{A+R}m_j,\qquad
\sum_{j\in\mathcal A}m_j+R=50.
\end{equation}
Here $m_j$ is the number of student steps executed in draft $j$. Only a terminal accepted draft can have fewer than three steps. Conditional on $R=8$, the 42 committed student steps therefore consist of exactly $A=14$ accepted drafts, and $N_{\mathrm{LTE}}=14+8=22$. The event acceptance rate is $14/22=63.6\%$, whereas the fraction of committed integration steps executed by the student is $42/50=84\%$. These are different denominators. If all rejected drafts also have length three, $N_S=66$; terminal clipping can reduce student work without changing the 22 routing events in this case. Thus eight teacher updates correspond to 22 LTE evaluations in the stated 50-step trajectory, not eight or 66 LTE evaluations. The budget-constrained random policy similarly has 22 decision events when $R=8$, but uses random decisions rather than requiring LTE network scores. This conditional identity does not establish the mean teacher usage of unconstrained LTE over the evaluation set.

\paragraph{Derived student-forward counts.} For the random control, the eight teacher updates leave 42 committed student steps. Executed student forwards additionally include every discarded draft:
\begin{equation}
N_S=42+\sum_{j=1}^{8}m_j,\qquad 1\leq m_j\leq3,
\end{equation}
where $m_j$ is the executed length of rejected draft $j$. This gives $N_S=66$ when all eight rejected drafts have length three; terminal clipping can reduce the count. For the reference LTE execution, full-length drafts similarly give $N_S=3(14+8)=66$ and $N_T=8$. These are conditional accounting values, not measured mean call counts or latency predictions.

\paragraph{Interpretation.} The dynamic comparison gives an FID difference of 1.47 at a latency difference of 0.001~s, supporting the usefulness of learned routing at the measured operating point. The 0.001~s difference is not interpreted as a speed advantage. Static controls are faster (0.335--0.348~s) but have higher FID (3.45--3.72). These measurements do not establish that LTE identifies high-curvature regions, nor do they isolate decision quality at exactly matched teacher usage. The routing-control table reports point estimates; variability across random-policy seeds is not quantified here.

\section{Additional Ablation Study}
\label{app:ablation}

Table~\ref{tab:ablation_additional} extends the main ablation with randomized-horizon variants. Unless stated otherwise, rows use ImageNet $256\times256$, $K=3$, and $\tau=0.80$.

\begin{table}[h]
  \centering
  \small
  \caption{Protocol-consistent extended ablation on ImageNet $256\times256$. Inference uses $K=3$, $\tau=0.80$; the randomized horizons refer to training.}
  \label{tab:ablation_additional}
  \setlength{\tabcolsep}{10pt}
  \begin{tabular}{lcccc}
    \toprule
    \textbf{Variant} & \textbf{Accept Rate \%} $\uparrow$ & \textbf{Speedup} $\uparrow$ & \textbf{FID} $\downarrow$ & \textbf{IS} $\uparrow$ \\
    \midrule
    \rowcolor{gray!10} \textbf{Full TamedFlow (default)} & 63.6 & 2.25$\times$ & \textbf{2.58} & \textbf{233.43} \\
    \midrule
    \textbf{Module and loss ablations} & \multicolumn{4}{c}{} \\
    \ \ w/o AlignNet & 46.2 & 1.87$\times$ & 2.95 & 211.34 \\
    \ \ w/o VAH & 39.3 & 1.63$\times$ & 3.23 & 205.71 \\
    \ \ w/o $\mathcal{L}_{\mathrm{roll}}$ & 54.2 & 2.02$\times$ & 2.71 & 230.15 \\
    \ \ w/o $\mathcal{L}_z$ warmup & 46.2 & 1.81$\times$ & 3.29 & 203.50 \\
    \ \ w/o $\mathcal{L}_z$ & 28.1 & 1.45$\times$ & 4.15 & 198.80 \\
    \midrule
    \textbf{$K$ strategy ablations} & \multicolumn{4}{c}{} \\
    \ \ Random $K\sim\mathcal{U}(1,3)$ & 75.0 & 2.14$\times$ & 2.53 & 233.22 \\
    \ \ Random $K\sim\mathcal{U}(1,5)$ (default) & 63.6 & 2.25$\times$ & 2.58 & 233.43 \\
    \midrule
    \ \ Pure Student (no LTE, no routing) & -- & 5.80$\times$ & 5.56 & 195.07 \\
    \bottomrule
  \end{tabular}
\end{table}

\begin{table}[h]
  \centering
  \small
  \caption{Same-student validation under the common DDIM-50 protocol. Only FID is reported because these direct adaptive baselines were evaluated as quality diagnostics, not as the headline latency comparison.}
  \label{tab:same_student_validation}
  \setlength{\tabcolsep}{6pt}
  \begin{tabular}{lc@{\qquad}lc}
    \toprule
    \textbf{Strategy} & \textbf{FID} & \textbf{Strategy} & \textbf{FID} \\
    \midrule
    FoCa & 10.40 & TaylorSeer & 11.92 \\
    SpeCa & 13.18 & Dense Student & 5.56 \\
    \bottomrule
  \end{tabular}
\end{table}

The extended rows reinforce the main ablation: representation alignment, residual correction, rollout loss, and endpoint loss each improve the usable routed operating point. Randomizing training over $K\sim\mathcal{U}(1,3)$ yields a nearby but distinct point (FID 2.53 at $2.14\times$), whereas the wider default $\mathcal{U}(1,5)$ reaches FID 2.58 at $2.25\times$.

At fixed $K=3$, the 75.0\% point corresponds to 15 accepted and 5 rejected events (60 student and 5 teacher calls under full-length drafts), whereas the 63.6\% point corresponds to 14 accepted and 8 rejected events (66 student and 8 teacher calls). Thus the call-count model predicts lower backbone work for the higher-acceptance point. Its measured $2.14\times$ versus $2.25\times$ ordering was obtained from separate checkpoint/runtime measurements, and the available component trace does not provide a matched per-call timing decomposition; we therefore do not interpret this ordering as a monotonic acceptance--latency law.

Under the common DDIM-50 diagnostic, directly attaching FoCa, TaylorSeer, or SpeCa to the compressed student degrades FID relative to dense student execution. This documents the Combination Gap for these implementations; it does not test whether the same methods remain fragile after trajectory training.

\begin{figure}[!htbp]
  \centering
  \includegraphics[width=0.9\linewidth]{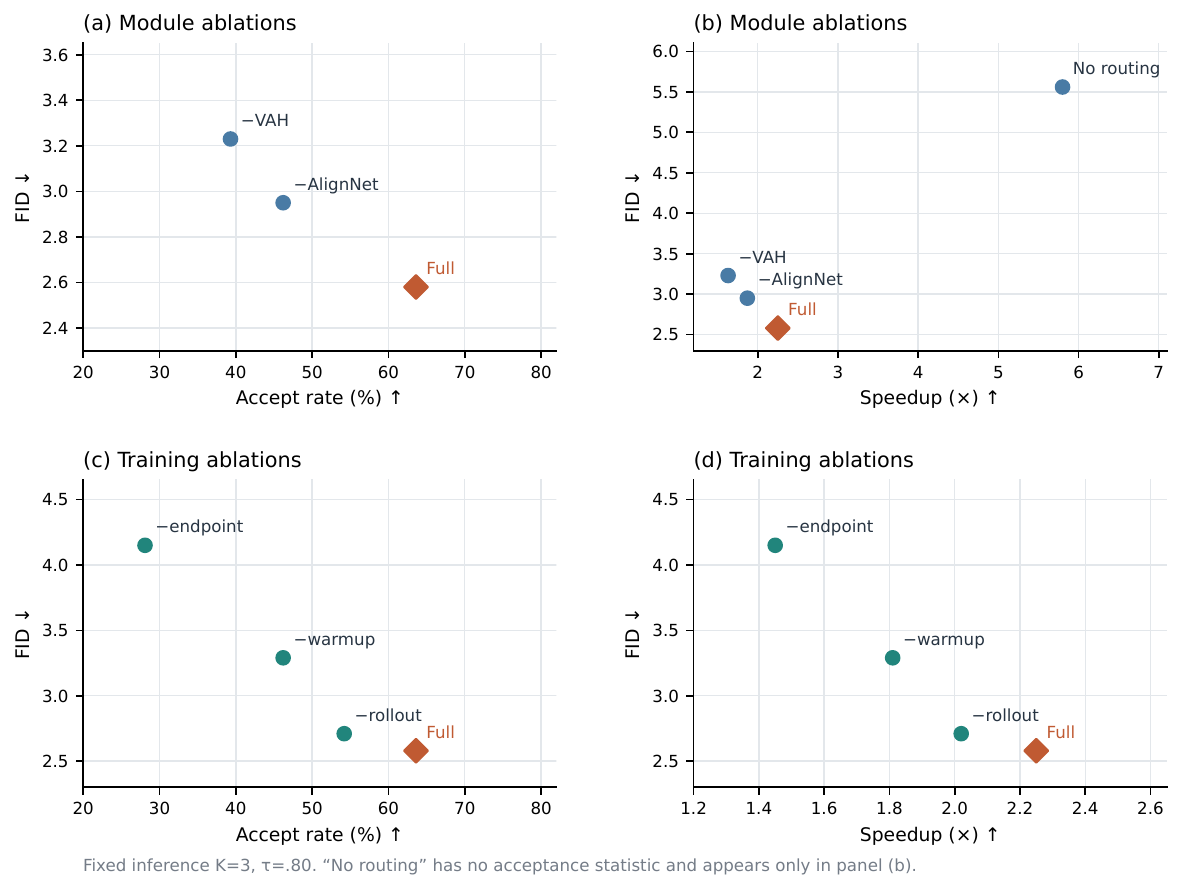}
  \caption{Module and training ablations from Tables~\ref{tab:abl} and~\ref{tab:ablation_additional}. The panels relate FID to acceptance or speedup at $K=3$, $\tau=0.80$. The no-routing student has no acceptance statistic and is omitted from the acceptance panels. Points are reported operating points; no fitted trend or uncertainty interval is implied.}
  \label{fig:ablation_multiples}
\end{figure}

\begin{table}[h]
  \centering
  \small
  \caption{Solver robustness on ImageNet $256\times256$. DPM-Solver++(2M) is evaluated as a handoff-safe higher-order solver diagnostic; the formal bound remains DDIM-specific.}
  \label{tab:solver_robustness}
  \setlength{\tabcolsep}{4pt}
  \begin{tabular}{llccccc}
    \toprule
    \textbf{Solver} & \textbf{Method} & \textbf{Latency} & \textbf{Speedup} & \textbf{Accept} & \textbf{FID} & \textbf{IS} \\
    \midrule
    Euler & Teacher & 1.044 & 1.00$\times$ & -- & 2.35 & 237.62 \\
    Euler & TamedFlow & 0.464 & 2.25$\times$ & 63.6\% & 2.58 & 233.43 \\
    DPM-Solver++(2M) & Teacher & 1.089 & 1.00$\times$ & -- & 2.25 & 244.12 \\
    DPM-Solver++(2M) & TamedFlow & 0.455 & 2.39$\times$ & 75.0\% & 2.44 & 237.55 \\
    \bottomrule
  \end{tabular}
\end{table}

For the DPM-Solver++(2M) diagnostic, the multistep history is associated with the current accepted state. When a student draft is rejected, the tentative student states and derivative history are discarded together; the sampler rolls back, applies one teacher rescue update, and starts a fresh solver history at the rescued state. No student derivative is mixed with a teacher derivative across a rejection boundary, and the first post-handoff update uses the solver's order-one startup rule before the two-step history is rebuilt. Thus Table~\ref{tab:solver_robustness} tests end-to-end robustness under an explicit reset policy; Proposition~\ref{prop:curvature_trap} continues to cover only the Euler analysis.

\begin{figure}[!htbp]
  \centering
  \includegraphics[width=0.78\linewidth]{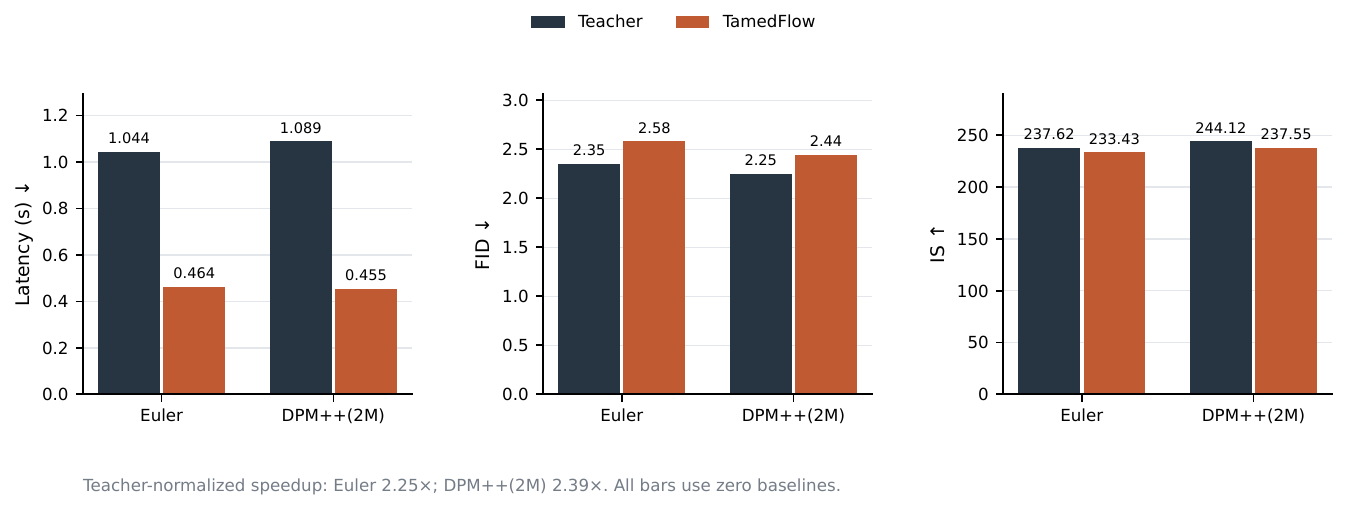}
  \caption{Paired teacher and TamedFlow measurements under Euler and DPM-Solver++(2M), reproducing Table~\ref{tab:solver_robustness}. Speedups are relative to the corresponding solver teacher. The DPM++ result uses the stated history-reset policy; the theoretical bound remains DDIM-specific.}
  \label{fig:solver_robustness}
\end{figure}

\begin{figure}[!htbp]
  \centering
  \includegraphics[width=0.8\linewidth]{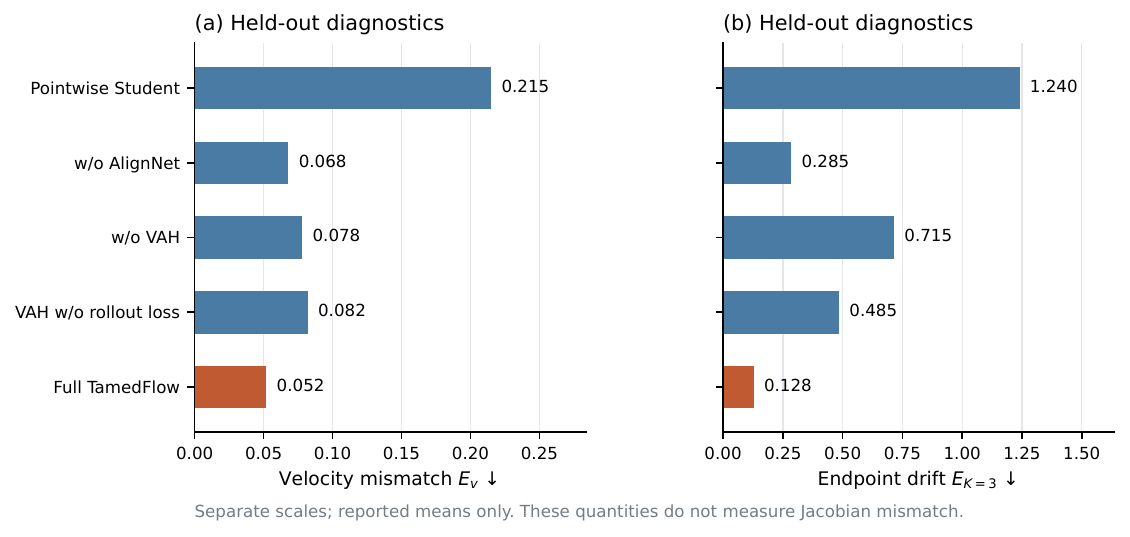}
  \caption{Held-out velocity mismatch and three-step endpoint drift by model variant. The panels replace the redundant numerical table and use separate zero-based axes. These means do not measure Jacobian mismatch or establish that training minimizes the theoretical bound constants; no uncertainty estimates are available.}
  \label{fig:bound_diagnostics}
\end{figure}

Figure~\ref{fig:bound_diagnostics} reports velocity and endpoint errors, not Jacobian alignment. Removing AlignNet or VAH changes the measured velocity or endpoint mismatch, while removing rollout supervision increases short-horizon drift.

Figure~\ref{fig:hyper_k} visualizes the impact of draft horizon $K$ and routing threshold $\tau$ on generation quality and inference efficiency. The four marginal panels show FID and speedup versus $\tau$ (top) and $K$ (bottom). These one-dimensional sweeps motivate the default operating point but do not establish a globally optimal Pareto frontier.

\begin{figure}[htbp]
  \centering
  \includegraphics[width=0.60\textwidth,trim=173bp 0 0 0,clip]{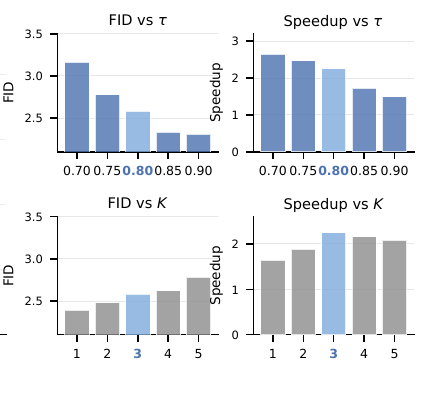}
  \caption{Impact of draft length $K$ and routing threshold $\tau$ on generation quality and inference efficiency. Columns: FID (left) and speedup (right). Rows: threshold $\tau$ (top) and draft length $K$ (bottom). The sweeps motivate the default $K=3$, $\tau=0.80$ configuration.}
  \label{fig:hyper_k}
\end{figure}

\section{Limitations}
\label{app:limitations}

While TamedFlow provides empirical gains across several settings, our findings are subject to several key limitations:

\paragraph{Theoretical scope of rollout analysis.} The discrete Gr\"onwall bound in Proposition~\ref{prop:curvature_trap} characterizes error accumulation under sequential Euler drafting to motivate multi-step trajectory supervision. It does not constitute a universal failure law for internal feature caching or forecasting methods, nor does it provide formal convergence or stability guarantees for higher-order ODE solvers such as DPM-Solver++.

\paragraph{Approximate nature of trajectory routing.} The Learnable Trajectory Estimator (LTE) operates as an empirical regression model trained on offline consistency targets rather than an exact, distribution-preserving verification mechanism. Rejection decisions are heuristic; false rejections incur latency penalties through redundant teacher evaluations, whereas false acceptances may allow minor perceptual artifacts to propagate into the generated sample.

\paragraph{Deployment and infrastructure overhead.} TamedFlow relies on a multi-stage offline training curriculum and maintains the teacher, student, VAH, and LTE models co-resident in GPU memory during inference ($\sim0.5$~GB VRAM overhead in our FP32 profile). As a result, it does not fully eliminate the teacher model from deployment. Furthermore, our reported speedups reflect batch-size-one latent generation on single GPUs; characterization under large-batch throughput, concurrent serving pipelines, and extreme memory-constrained edge environments remains an avenue for future work.

\section{LLM Usage Statement}
Large language models were used for language editing and consistency checks across the manuscript. Experimental results are author-provided; the authors are responsible for checking the final equations, measurements, citations, and implementation details.

\end{document}